\documentclass[twoside,11pt]{article}
\pdfoutput=1
\usepackage{jair, theapa, rawfonts}
\usepackage{amsthm}
\usepackage[noend]{algorithmic}
\usepackage{algorithm}
\usepackage[usenames]{color} 
\usepackage{url}
\usepackage{graphicx}

\ShortHeadings{Distributed Heuristic Forward Search for Multi-Agent Systems}
{Nissim \& Brafman}
\firstpageno{1}

\definecolor{ToDoColor}{rgb}{0.1,0.2,1}

\newcommand{\astar}{\textsc{a*}}
\newcommand{\mafs}{\textsc{mafs}}
\newcommand{\masas}{\textsc{mad-a*}}
\newcommand{\mastrips}{\textsc{ma-strips}}
\newcommand{\strips}{\textsc{strips}}

\newcommand{\mafd}{\textsc{ma-fd}}
\newcommand{\ppastar}{\textsc{pp-a*}}
\newcommand{\prune}{\rho}

\begin{document}

\title{Distributed Heuristic Forward Search for Multi-Agent Systems}

\author{\name Raz Nissim \email raznis@cs.bgu.ac.il \\
       \name Ronen Brafman \email brafman@cs.bgu.ac.il \\
       \addr Ben-Gurion University of the Negev,\\
       Be'er Sheva, Israel}


\maketitle

\begin{abstract}
This paper describes a number of distributed forward search algorithms for solving multi-agent planning problems. We introduce a distributed formulation of non-optimal forward search, as well as an optimal version, \masas. Our algorithms exploit the structure of multi-agent problems to not only distribute the work efficiently among different agents, but also to remove symmetries and reduce the overall workload. The algorithms ensure that private information is not shared among agents, yet computation is still efficient -- outperforming current state-of-the-art distributed planners, and in some cases even centralized search -- despite the fact that each agent has access only to partial information. 
\end{abstract}

\section{Introduction}
\label{Introduction}
Interest in multi-agent systems is constantly rising, and examples of virtual and real systems abound, with virtual social communities providing many such instances. The ability to plan for such systems and the ability of such systems to autonomously plan for themselves is an important challenge for AI, especially as the size of these systems can be quite large.
In this context, a fundamental question is how to perform distributed planning for a distributed multi-agent system efficiently, and in many cases, how to do it while preserving privacy. 

Distributed planning is interesting for a number of reasons. Scientifically and intellectually, it is  interesting to seek distributed algorithms for fundamental computational tasks, such as classical planning. Similarly, it is  interesting (and likely very useful in the long term) to seek distributed versions of fundamental tools in computer science, and search is definitely such a tool. Moreover, there are pragmatic reasons
for seeking distributed algorithms. As an example, imagine a setting in which different manufacturers or service providers can publish their capabilities and then collaborate with each other to provide new products or services that none of them can provide alone. Such providers will certainly 
need to reveal some sort of public interface, describing what they can contribute to others, as well as what they require from others.
But most likely, they will not want to describe their inner workings: their internal state and how they can manipulate it (e.g., their current stock levels, machinery, logistics capabilities, personnel, other commitments, etc.). This is usually confidential proprietary information that an agent would not want to reveal, although clearly one must reason about it during the planing process. 

In principle, the above problem can be addressed using a central trusted party running a suitable planning algorithm.
However, such a trusted party may not exist in all settings. Moreover, centralized planning puts the entire computational burden on a single agent, rather than distribute it across the system. Thus, centralized algorithms are less robust to agent failures, and sometimes less efficient.
For these reasons, distributed algorithms are often sought, and in our case in particular, distributed, privacy preserving algorithms.
Indeed, this is the main motivation for the field of distributed algorithms, and in particular, the work on distributed constraint satisfaction problems (CSP) \cite{YokooDIK98,MeiselsBook}. 

Yet another motivation for distributed algorithms is provided by planning domains in which search operators that correspond to actions are implemented using complex simulation software that is accessible to the relevant agent only because that agent is not interested in sharing it (due to privacy concerns or commercial interests) or because it is not realistic to transfer, implement, and appropriately execute such software as part of a planning algorithm. As an example, consider planning by a group of robotic agents, each
with different capabilities. Each agent has a simulator that can compute the effect of its actions, which the agents do not want to share with
each other. Thus, the application of each agent's actions during  search can only be done by the agents themselves

Moreover, it is often the case that good distributed algorithms formulated for a cooperative team provide the foundation for algorithms and mechanisms for solving similar problems for teams of self-interested agents. For example, the work on planning games \cite{BrafmanDET09,BrafmanDET10} suggests modified versions of an earlier algorithm for cooperative multi-agent systems~\cite{Brafman200828} and work on mechanism design for solving distributed CSPs by self-interested agents \cite{PetcuFP08} is based on earlier work in distributed CSPs for cooperative teams \cite{PetcuF05}. Finally, work on distributed algorithms can lead to insights on problem factoring and abstraction, as we shall demonstrate later on in this paper -- presenting a new effective pruning technique for centralized planning that is an outgrowth of our work on distributed search.


There is a long tradition of work on multi-agent planning for cooperative and non-cooperative agent teams involving centralized and distributed algorithms, often using involved models that model uncertainty, resources, and more~\cite{Nilsson1980,HansenZ01a,BernsteinGIZ02,SzerCZ05}, and much work on how to coordinate the local plans of agents or to allow agents to plan locally under certain constraints~\cite{CoxD05,SteenhuisenWMV06,MorsVW04,MorsW05}.
However, our starting point is a more basic, and hence, we believe, more
fundamental model introduced by Brafman and Domshlak (BD) which offers what is possibly the simplest model of MA planning -- \mastrips\ \cite{Brafman200828}. \mastrips\ minimally extends standard \strips\ (or PDDL) models by specifying a set of agent ids, and associating each action in the domain with one of the agents. Thus, essentially, it partitions the set of actions among the set of agents. We believe that this model serves as a skeleton model for most work on multi-agent planning, and that the insights gained from it can help us address more involved models as well. 

Distributed planning can easily be performed in \mastrips\ using existing distributed planning algorithms~\cite{VrakasRV01,KishimotoFB09,BurnsLRZ10} applied to the underlying \strips\ planning problem (i.e., where we ignore agent identities). However, these algorithms were devised to speed-up the solution of centralized planning problems given access to a distributed computing environment, such as a large cluster. However, these algorithms do not "respect" the inherent distributed form of the problem, giving all agents access to all
actions and hence do not preserve privacy. Nor can they be used when search operators cannot be shared by agents, as in the scenario described earlier.

Recently, a number of algorithms that maintain agent privacy and utilize the inherent distributed structure of the system have emerged.
The most natural approach is based on distributed CSP techniques and was introduced in BD's original work. BD formulate a particular
CSP that is particularly suited for \mastrips\ problems whose solution is a plan. This algorithm can be transformed into a fully distributed algorithm simply by using a distributed CSP solver. Unfortunately, distributed CSP solvers cannot handle even the smallest instances of MA planning problems. Consequently, a dedicated algorithm, based on the ideas of BD, \emph{Planning-First}, was developed \cite{NissimBD10}. While performing well on some domains, this algorithm had trouble scaling up to problems in which each agent had to execute more than a small number of actions. (Indeed, BD's algorithm scales exponentially with the minimal number of actions per agent in the solution plan.) Recently, a new, improved algorithm, based on partial-order planning, \textsc{map-pop},  was developed by \cite{TorrenoOS12}. Yet, this algorithm, too, leaves a serious gap between what we can solve using a distributed planner and what can solved using a centralized planner. Moreover, \emph{neither algorithm attempts to generate a cost-optimal plan}.

In single-agent planning, constraint-based and partial-order planning techniques are currently dominated by heuristic forward search techniques. Thus, it is natural to ask whether it is possible to formulate a distributed heuristic forward search algorithm for distributed planning. This paper provides a positive answer to this question in the form of a general approach to distributed search in which each agent performs
only the part of the state expansion relevant to it. The resulting algorithms are very simple and very efficient -- outperforming previous algorithms by orders of magnitude -- and offer similar flexibility to that of forward-search based algorithms for single-agent planning.
They respect the natural distributed structure of the system, and thus allow us to formulate privacy preserving versions.

One particular variant of our general approach yields a distributed version of the \astar\ algorithm, called \masas, using which we obtain a general, efficient distributed algorithm for \emph{optimal} planning. \masas\ solves a more difficult problem than centralized search because in
the privacy preserving setting, each agent has less knowledge than a centralized solver. Yet, it is able to solve some problems centralized \astar\ cannot solve. As we will show, the main reason for this speed-up is an interesting optimality preserving pruning technique that is naturally
built into our search approach. This insight has led to to a new effective pruning technique for centralized  search that we shall describe later on.

The rest of this paper is organized as follows. The next section presents the model we use and related work. Section \ref{FSalgorithm} describes our MA forward search algorithm, and Section \ref{sec:optimal} describes \masas, a modified version which maintains optimality. We next present the MA planning framework \mafd, and empirical results for both \mafs\ and \masas. Section \ref{sec:partition} shows how we
can exploit insights from our distributed search methods to obtain a new effective pruning methods for centralized search. Section \ref{discussion} concludes the paper with a discussion.

\section{Background}
\label{background}


A \mastrips\ problem \cite{Brafman200828} for a set of agents $\Phi=\{\varphi_i\}_{i=1}^{k}$, is given by a 4-tuple $\Pi = \langle P,\{A_i\}_{i=1}^{k},I,G \rangle$, where $P$ is a finite set of propositions, $I\subseteq P$ and $G\subseteq P$ encode the initial state and goal, respectively, and for $1\leq i\leq k$, $A_i$ is the set of actions agent $\varphi_i$ is capable of performing. Each action $a=\langle \mathrm{pre}(a),\mathrm{eff}(a) \rangle$ is given by its preconditions and effects. A {\em plan\/} is a solution to $\Pi$ iff it is a solution to the underlying \strips\ problem obtained by ignoring the identities of the agent associated with each action. Since each action is associated with an agent, a plan tells each agent what to do and when to do it. In different planning contexts, one  might seek special types of solutions. For example, in the context of planning games \cite{BrafmanDET09}, \emph{stable} solutions are sought. We focus on cooperative multi-agent systems, seeking either a (standard) solution or a cost-optimal solution.


The partitioning of the actions to agents yields a distinction between private and public propositions and actions. A \textit{private} proposition of agent $\varphi$ is required and affected only by the actions of $\varphi$. An action is \textit{private} if all its preconditions and effects are private. All other actions are classified as \textit{public}. That is, $\varphi$'s private actions affect and are affected only by $\varphi$'s actions, while its public actions may require or affect the actions of other agents. For ease of the presentation of our algorithms and their proofs, we assume that all actions that achieve a goal condition are considered \textit{public}. Our methods are easily modified to remove this assumption.

We note that while the notion of private\slash public is natural to the \mastrips\ encoding, it can easily be applied in models having multi-valued variables. For example, in SAS+, where each variable may have multiple values, the analogous of a (boolean) proposition in \mastrips\ is a $\langle \mathrm{variable},\mathrm{value} \rangle$ pair. Such a pair is considered private if it is required, achieved or destroyed only by the actions of a single agent. Consequently, actions which require, achieve or destroy only private $\langle \mathrm{variable},\mathrm{value} \rangle$ pairs are considered private. For clarity and consistency with previous work we use \mastrips\ notation when discussing the theoretical aspects of our work. However, the examples given, as well the practical framework we present for MA planning, use the more concise multi-valued variables SAS+ encoding.

In a distributed system of fully-cooperative agents privacy is not an issue, and so the distinction between private and public actions is not essential, although it can be exploited for computational gains~\cite{Brafman200828}. However, there are settings in which agents collaborate on a specific task, but prefer not to reveal private information about their local states, their private actions, and the cost of these private actions.
They wish only to make their public interface known -- i.e., the public preconditions and effects of their actions. 

This setting is the planning equivalent to the area of distributed CSPs, where agents must coordinate (e.g., schedule a meeting) while keeping certain constraints and private variables private. We will refer to algorithms that plan without revealing this information as \emph{privacy preserving} (distributed) planning algorithms. More specifically, in a privacy-preserving algorithm the only information available about an agent to others is its set of public actions, projected onto public propositions. This can be viewed as the interface between the agents. Information about an agent's private actions and private aspects of a public action are known to the agent only. 

Given a model of a distributed system such as \mastrips, it is natural to ask how to search for a solution. The best known example of distributed search is that of distributed CSPs \cite{YokooDIK98}, and various search techniques and heuristics have been developed for it \cite{MeiselsBook}. Planning problems can be cast as CSP problems (given some bound on the number of actions), and the first attempt to solve \mastrips\ problems was based on a reduction to distributed CSPs. More specifically, Brafman and Domshlak  introduced the \textit{Planning as CSP+Planning} methodology for planning by a system of cooperative agents with private information. This approach separates the public aspect of the problem, which involves finding public action sequences that satisfy a certain distributed CSP, from the private aspect, which ensures that each agent can actually execute these public actions in a sequence. Solutions found are locally optimal, in the sense that they minimize $\delta$, the maximal number of public actions performed by an agent. This methodology was later extended to the first fully distributed MA algorithm for \mastrips\ planning, \textit{Planning-First} \cite{NissimBD10}. \textit{Planning First} was shown to be efficient in solving problems where the agents are very loosely coupled, and where $\delta$ is very low. However, it does not scale up as $\delta$ rises, mostly due to the large search space of the distributed CSP. Recently, a distributed planner based on partial order planning was introduced \cite{TorrenoOS12}, which outperforms Planning First, effectively solving more tightly coupled problems. Both methods are privacy preserving, but do not guarantee cost-optimal solutions.

\section{Multi-Agent Forward Search}
\label{FSalgorithm}

This section describes our distributed variant of forward best-first search, which we call \mafs. We begin with the algorithm itself, including an overview and pseudo-code. We next provide an example of the flow of \mafs, and a discussion of its finer points.

\subsection{The MAFS Algorithm}
\label{the-algorithm}

Algorithms \ref{alg:forward-search}-\ref{alg:expand} depict the \mafs\ algorithm for agent $\varphi_i$. 
In \mafs, a separate search space is maintained for each agent. Each agent maintains an \textit{open list} of states that are candidates for expansion and a \textit{closed list} of already expanded states. It expands the state with the minimal $f$ value in its open list. When an agent expands state $s$, {\em it uses its own operators only.} This means two agents expanding the same state will generate \textit{different} successor states. 

Since no agent expands all relevant search nodes, messages must be sent between agents, informing one agent of open search nodes relevant to it expanded by another agent. Agent $\varphi_i$ characterizes state $s$ as relevant to agent $\varphi_j$ if $\varphi_j$ has a public operator whose public preconditions (the preconditions $\varphi_i$ is aware of) hold in $s$, and the creating action of $s$ is public. In that case, Agent $\varphi_i$ will send $s$ to Agent $\varphi_j$. 

\begin{algorithm}
\caption{\mafs\ for agent $\varphi_i$}
    \label{alg:forward-search}
\begin{algorithmic}[1]  
\WHILE{did not receive \textbf{true} from a solution verification procedure}
	\FORALL{messages $m$ in message queue}
		\STATE \textbf{process-message($m$)}
	\ENDFOR
	\STATE $s\leftarrow$ \textbf{extract-min}(open list) \label{line:extract_min}
	\STATE \textbf{expand($s$)}
\ENDWHILE
\end{algorithmic}
\end{algorithm}
\begin{algorithm}
\caption{process-message($m=\langle s,g_{\varphi_j}(s),h_{\varphi_j}(s)\rangle$)}
    \label{alg:process-message}
\begin{algorithmic}[1]  
\IF{$s$ is not in open or closed list \textbf{or} $g_{\varphi_i}(s)>g_{\varphi_j}(s)$  } 
	\STATE add $s$ to open list \textbf{and} calculate $h_{\varphi_i}(s)$
	\STATE $g_{\varphi_i}(s)\leftarrow g_{\varphi_j}(s)$
	\STATE $h_{\varphi_i}(s)\leftarrow max(h_{\varphi_i}(s),h_{\varphi_j}(s))$\label{line:update-h}
\ENDIF
\end{algorithmic}
\end{algorithm}
\begin{algorithm}
\caption{expand($s$)}
    \label{alg:expand}
\begin{algorithmic}[1]  
\STATE move $s$ to closed list
\IF{$s$ is a goal state}
	\STATE broadcast $s$ to all agents \label{line:broadcast1}
	\STATE initiate verification of $s$ as a solution	\label{line:broadcast2}
	\RETURN
\ENDIF
\FORALL{agents $\varphi_j \in \Phi$}\label{line:sending-forloop}
	\IF{the last action leading to $s$ was public \textbf{and} $\varphi_j$ has a public action for which all public preconditions hold in $s$} \label{line:relevant}
		\STATE send $s$ to $\varphi_j$	\label{line:send}
	\ENDIF
\ENDFOR
\STATE apply $\varphi_i$'s successor operator to $s$ \label{line:succ}
\FORALL{successors $s'$}\label{line:successors-forloop}
	\STATE update $g_{\varphi_i}(s')$ and calculate $h_{\varphi_i}(s')$
	\IF{$s'$ is not in closed list \textbf{or} $f_{\varphi_i}(s')$ is now smaller than it was when $s'$ was moved to closed list}
		\STATE move $s'$ to open list \label{line:move-to-open-list}
	\ENDIF
\ENDFOR
\end{algorithmic}
\end{algorithm}


The messages sent between agents contain the full state $s$, i.e.\ including both public and private variable values, as well as the cost of the best plan from the initial state to $s$ found so far, and the sending agent's heuristic estimate of $s$.
 When agent $\varphi$ receives a state via a message, it checks whether this state exists in its open or closed lists. If it does not appear in these lists, it is inserted into the open list. If a copy of this state with higher $g$ value exists, its $g$ value is updated, and if it is in the closed list, it is reopened. Otherwise, it is discarded. Whenever a received state is (re)inserted into the open list, the agent computes its local $h$ value for this state, and then can choose between\slash combine the value it has calculated  and the $h$ value in the received message. If both heuristics are known to be admissible, for example, the agent could choose the maximal of the two estimates, as is done in Line \ref{line:update-h} of Algorithm \ref{alg:process-message}.

Once an agent expands a \emph{solution} state $s$, it sends $s$ to all agents and awaits their confirmation. For simplicity, and in order to avoid deadlock, once an agent either broadcasts or confirms a solution, it is not allowed to create new solutions. If a solution is found by more than one agent, the one with lower cost is chosen, and ties are broken by choosing the solution of the agent having the lower ID. When the solution is confirmed by all agents, the agent initiates the trace-back of the solution plan. This is also a distributed process, which involves all agents that perform some action in the optimal plan. The initiating agent begins the trace-back, and when arriving at a state received via a message, it sends a trace-back message to the sending agent. This continues until arriving at the initial state. When the trace-back phase is done, a terminating message is broadcasted and the solution is outputted.

As we will see, this general and simple scheme -- apply your own actions/operators only and send relevant generated nodes to other agents -- can be used to distribute other search algorithms. However, there are various subtle points pertaining to message sending and termination that influence the correctness and efficiency of the distributed algorithm, which we discuss later.

To better demonstrate the flow of the algorithm, consider the example given in Figure \ref{fig:example}. In this example,  we have two agents who must cooperate in order to achieve the goal. The agents' actions are described on the left-hand side, where every node in the graph depicts an action, and an edge $(u,v)$ indicates that $u$ either achieves or destroys a precondition of $v$. There are two public actions $a_5,a_8$, which affect\slash depend on the only public variable, $v_4$, while the rest of the actions are private. In the initial state, all variable values are zero (i.e., $I=0000$), and the goal is $G=\{v_4=2\}$. When the agents begin searching, each applies its own actions only. Therefore, agent 2 quickly exhausts its search space, since as far as it's concerned, state $0020$ is a dead end. Agent 1 generates its search space, until it applies public action $a_5$, which results in state $s=2201$. $s$ is then sent to agent 2, since all the public preconditions of $a_8$ hold in $s$ (Line \ref{line:relevant} of Algorithm \ref{alg:expand}). Upon receiving $s$, agent 2 continues applying its actions, eventually reaching the goal state, which is then broadcasted.

\begin{figure*}[t]
  \caption{Description of the actions of an example planning problem, its reachable search space, and the search space generated by \mafs. Actions are represented as $<\textit{pre},\textit{eff}>$ and states are denoted by the values of variables $v_1,v_2,v_3,v_4$ respectively (For example, $1122$ denotes the state where $v_1=1,v_2=1,v_3=2,v_4=2$.).}
\label{fig:example}
\hspace{-40pt}
  \centering
    \includegraphics[width=1.025\textwidth]{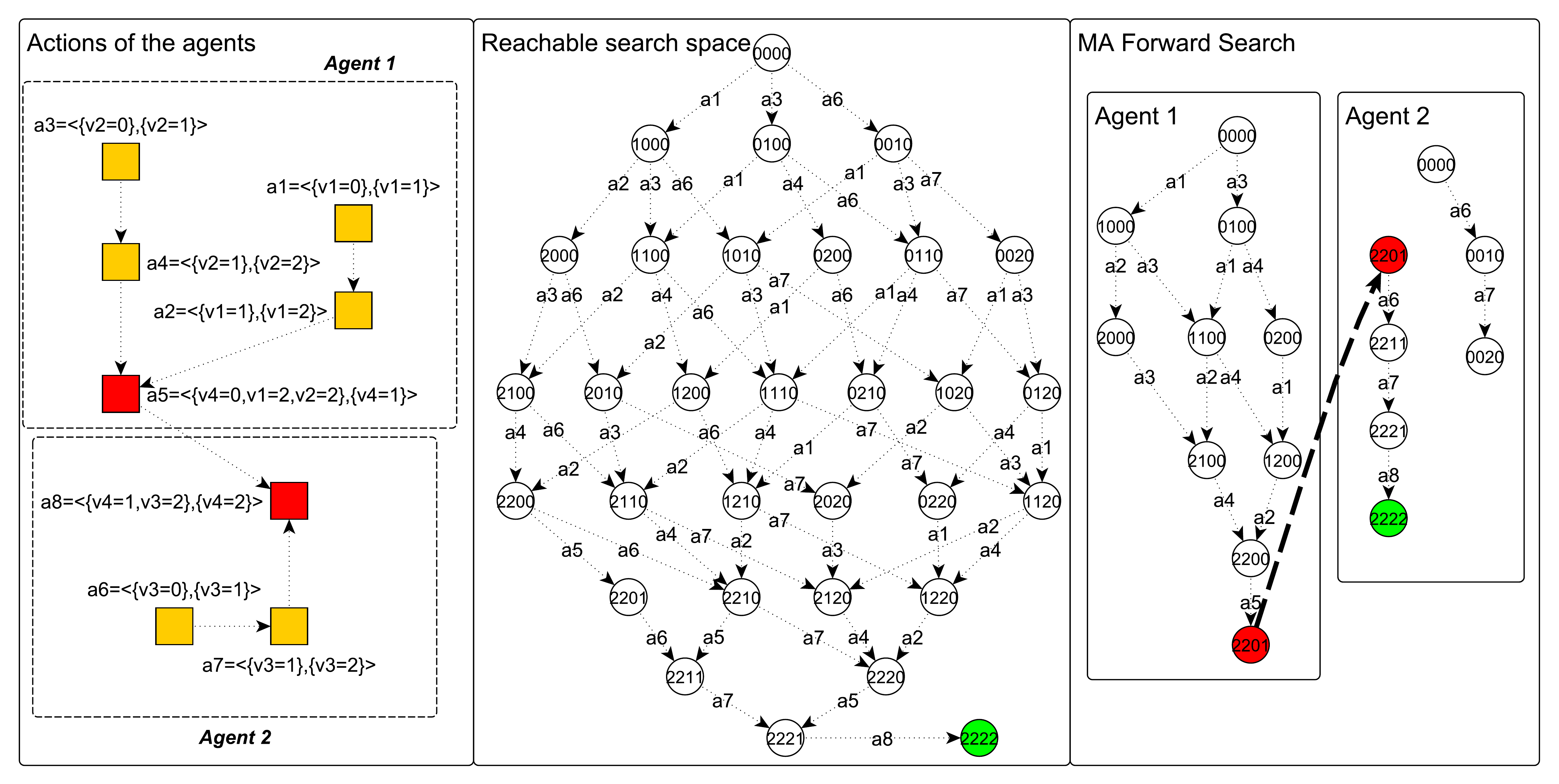}
\hspace{-40pt}
\end{figure*}

\subsection{Discussion}
We now discuss some of the more subtle points of \mafs.

\subsubsection{Preserving Agent Privacy}
\label{sec:privacy}
If our goal is to preserve privacy, it may appear that \mafs\ agents are revealing their private data because they transmit their private state in their messages. Yet, in fact, this information is not used by any of the other agents, nor is it altered. It is  simply copied to future states to be used only by the agent. 
Since private state data is used only as an ID, the agents can encrypt this data and keep a table locally, which maps IDs to private states. If this encryption can generate multiple possible IDs for each private state, other agents cannot identify other agents' private states.
 The issue of privacy is discussed further in Section \ref{sec:discussion_privacy}.

To compute heuristic estimates of states it receives, an agent must assess the effort required to achieve the goal from them.
To do this, it needs  some information about the
effort required of other agents to construct their part of the plan. 
In a fully cooperative setting, an agent can have access to the full
description of other agents' actions. In the privacy preserving setting, two issues arise. First, agents have only partial information about
other agents' capabilities -- they only have access to their public interface. Second, different agents may compute different heuristic
estimates of the same state because each agent has full information about its capabilities, but not about those of the others.
This issue does not affect the actual algorithm, which is agnostic to how agents compute their heuristic estimate,
although the fact that agents have less information can lead to poorer heuristic estimates. On the other hand, agents are free to
use different heuristic functions, and as we will demonstrate empirically,
 using the public interfaces only, we are still able to efficiently solve planning problems.



When a state is reached via a message, it includes the sending agent's heuristic estimate. Therefore, the receiving agent now has two (possibly different) estimates it can use. If the heuristics are known to be admissible, then clearly the maximal (most accurate) value is taken, as in line \ref{line:update-h} of Algorithm \ref{alg:process-message}. If not, the agent is free to decide how to use these estimates, depending on their known qualities.

\subsubsection{Relevancy and Timing of the Messages}
State $s$ is considered relevant to agent $\varphi_j$ if it has a public action for which all public preconditions hold in $s$ and the last action leading to $s$ was public (line \ref{line:relevant} of Algorithm \ref{alg:expand}). This means that all states that are products of private actions are considered irrelevant to other agents. As it turns out, since private actions do not affect other agent's capability to perform actions, an agent must send only states in which the last action performed was public, in order to maintain completeness (and optimality, as proved in the next section). Regarding states that are products of private actions as irrelevant decreases communication, while effectively pruning large, symmetrical parts of the search space. In fact, we will show in Section \ref{sec:partition} how this property of \mafs\ can be used to obtain state-space pruning in
centralized planning algorithms, using a method called \emph{Partition-based pruning}. 

As was hinted earlier, there exists some flexibility regarding when these relevant states are sent. Centralized search can be viewed as essentially ``sending'' every state (i.e., inserting it to its open list) once it is generated. In \mafs, relevant states can be sent when they are expanded (as in the pseudo-code) or once they are generated (changing Algorithm \ref{alg:expand} by moving the for-loop on line \ref{line:sending-forloop} inside the for-loop on line \ref{line:successors-forloop}). The timing of the messages is especially important in the distributed setting since agents may have different heuristic estimations. Sending the messages once they are generated increases communication, but allows for states that are not considered promising by some agent to be expanded by another agent in an earlier stage. Sending relevant states when they are expanded, on the other hand, decreases communication, but delays the sending of states not viewed as promising. Experimenting with the two options, we found that the lazy approach, of sending the messages only when they are expanded, dominates the other, most likely because communication can be costly.

\subsubsection{Robustness}
One of the main motivations for distributed problem solving is \emph{robustness}. We discuss robustness with respect to agent failure, or the ability of the algorithm to handle the failure of one or more of the computing agents. In the centralized case, where only one computing agent exists, its failure means the inability to solve the problem. In the case of \mafs, the algorithm can still find solutions even if a group of agents fails to perform their computation. To do this, the agents simply need to ignore all states in which the failing agent participated in the plan leading up to them. If this is done, the algorithm will find a solution excluding the failed agent.

For this to be done, the agents must additionally identify each state \emph{s} with the set of agents participating in  plans leading to up to \emph{s}. If \emph{s} can be reached via two paths having different participating agents sets, \emph{s} is duplicated, in order to maintain completeness. If agent $\varphi$ fails, all agents remove from their open lists all states in which $\varphi$ is a participating agent, and ignore any such states arriving in future messages. This simple alteration of \mafs\ guarantees that if a solution excluding the failed agent exists, it is found.

\subsubsection{Search Using Complex Actions}


In Section \ref{Introduction}, we mentioned a scenario where search operators corresponding to real-world actions are implemented using complex simulation software. This situation can arise, for example, with a team of heterogeneous robotic agents, each of which has a dedicated simulator of its actions. There are natural cases in which such agents are unlikely to want to share such generative models. For example, imagine a robotic team where different robots are supplied by different manufacturers, e.g., some are autonomous heavy equipment, while others are humanoid robots, and yet others are drones. Simulation software for such robots is usually complex and proprietary. It is unlikely that agents would want to share it, although they most likely have no problem advertising their capabilities. Moreover, in the case of ad-hoc teams, transferring and installing such software is unlikely to work online. (Any person who attempted to install such simulators knows how sensitive they can be to the particular computing environment they run on). 

Our approach is well suited for such settings: First, forward search methods are capable of using generative, rather than declarative models of the agent's actions, as their central steps involves the generation of successor states and their insertion into appropriate queues. They are oblivious as to how the operators are described or implemented, as long as successor states can be generated. Second, our approach respects the natural system structure, and each agent need only apply its own operators. Thus, there is no need to share the generative models amongst the agents.

One problem, however, with generative operators is the fact that most contemporary methods for generating heuristic functions for generated states require either a declarative \strips-like description or a generative model (in the case of sampling methods). Fortunately, our empirical results indicate that the use of an approximate model works quite well in practice. Indeed, our approach assumes that other agents use only the public part of an agent's action model, which is only an approximation. Even if the original action model is generative, a declarative approximate model can be constructed using learning techniques~\cite{YangWJ07}. Alternatively, sampling methods could use a suitably developed simplified simulator.

\section{Optimal MAFS}
\label{sec:optimal}
\mafs\, as presented, is not an optimal planning algorithm. It can, however, be slightly modified in order to achieve optimality. We now describe these modifications, which result in a MA variation of \astar\ we refer to as Multi-Agent Distributed \astar\ (\masas).

As in \astar, the state chosen for expansion by each agent must be the one with the lowest $f=g+h$ value in its open list, where the heuristic estimates are admissible. In \masas, therefore, \textit{extract-min} (Line \ref{line:extract_min} in Algorithm \ref{alg:forward-search}) must return this state. 

\subsection{Termination Detection}
Unlike in \astar, expansion of a goal state in \mafs\ does not necessarily mean an optimal solution has been found. In our case, a solution is known to be optimal only if all agents prove it so. Intuitively, a solution state $s$ having solution cost $f^*$ is known to be optimal if there exists no state $s'$ in the open list or the input channel of some agent, such that $f(s')<f^*$. In other words, solution state $s$ is known to be optimal if $f(s)\leq f_{\mathrm{lower-bound}}$, where $f_{\mathrm{lower-bound}}$ is a lower bound on the $f$-value of the entire system (which includes all states in all open lists, as well as states in messages that have not been processed, yet). 

To detect this situation, we use Chandy and Lamport's \textit{snapshot algorithm} \cite{ChandyL85}, which enables a process to create an approximation of the global state of the system, without ``freezing'' the distributed computation. Although there is no guarantee that the computed global state actually occurred, the approximation is good enough to determine whether a stable property currently holds in the system. A property of the system is \textit{stable} if it is a global predicate which remains true once it becomes true. Specifically, properties of the form $f_{\mathrm{lower-bound}}\geq c$ for some fixed value $c$, are stable when $h$ is a \textit{globally consistent} heuristic function. That is, when $f$ values cannot decrease along a path. In our case, this path may involve a number of agents, each with its $h$ values. If each of the local functions $h_\varphi$ are consistent, and agents apply the $\max$ operator when receiving a state via a message (known as \emph{pathmax}), this property holds\footnote{Although recent work \cite{Holte10} shows that pathmax does not necessarily make a \emph{bona-fide} consistent heuristic, pathmax does ensure that $f$-values along a path are non-decreasing.}.

We note that for simplicity of the pseudo-code we omitted the detection of a situation where a goal state does not exist. This can be done by determining whether the stable property \textit{``there are no open states in the system''} holds, using the same \textit{snapshot algorithm}.

\subsection{Proof of Optimality}
\label{optimality}

\newtheorem{lemma}{Lemma}
\newtheorem{theorem}{Theorem}
\newtheorem{corollary}{Corollary}

We now prove the optimality of \masas. We must note that as it is presented, \masas\ maintains completeness (and optimality) only if all actions which achieve some goal condition are considered public. This property is assumed throughout this section, but the algorithm is easily modified to remove it. We begin by proving the following lemmas regarding the solution structure of a MA planning problem. 
\begin{lemma}
\label{lem:flipping}
Let $P=(a_1,a_2\dots,a_{k})$ be a legal plan for a MA planning problem $\Pi$. Let $a_i,a_{i+1}$ be two consecutive actions taken in $P$ by different agents, of which at least one is private. Then $P'=(a_1,\dots,a_{i+1},a_i,\dots,a_{k})$ is a legal plan for $\Pi$ and $P(I)=P'(I)$.
\end{lemma}

\begin{proof}
By definition of private and public actions, and because $a_i,a_{i+1}$ are actions belonging to different agents, $\mathrm{varset}(a_i)\cap \mathrm{varset}(a_{i+1}) = \emptyset$, where $\mathrm{varset}(a)$ is the set of variables which affect and are affected by $a$. Therefore, $a_i$ does not achieve any of $a_{i+1}$'s preconditions, and $a_{i+1}$ does not destroy any of $a_i$'s preconditions. Therefore, if $s$ is the state in which $a_i$ is executed in $P$, $a_{i+1}$ is executable in $s$, $a_i$ is executable in $a_{i+1}(s)$, and $a_i(a_{i+1}(s))=a_{i+1}(a_i(s))$. Therefore, $P'=(a_1,\dots,a_{i+1},a_i,\dots,a_{k})$ is a legal plan for $\Pi$. Since the suffix $(a_{i+2},a_{i+3},\dots,a_{k})$ remains unchanged in $P'$, $P(I)=P'(I)$, completing the proof.
\end{proof}

\begin{corollary}	\label{lem:restriction-on-P}
For a MA planning problem $\Pi$ for which an optimal plan
$P=(a_1,a_2,\dots,a_{k})$ exists, there exists an optimal plan
$P'=(a_1',a_2',\dots,a_{k}')$ for which the following restrictions apply:
\begin{enumerate}
\item If $a_i$ is the first public action in $P'$, then $a_1,\dots,a_i$ belong to the same agent.
\item For each pair of consecutive public actions $a_i,a_j$ in $P'$, all actions $a_l$, $i<l\leq j$ belong to the same agent.
\end{enumerate}
\end{corollary}

\begin{proof}
Using repeated application of Lemma \ref{lem:flipping}, we can move any ordered sequence of private actions performed by agent $\varphi$, so that it would be immediately before $\varphi$'s subsequent public action and maintain legality of the plan. Since application of Lemma \ref{lem:flipping} does not change the cost of plan, the resulting plan is cost-optimal as well.
\end{proof}

Next, we prove the following lemma, which is a MA extension to a well known result for \astar. In what follows, we have tacitly assumed a \textit{liveness} property with the conditions that every sent message eventually arrives at its destination and that all agent operations take a finite amount of time. Also, for the clarity of the proof, we assume the atomicity of the \textit{expand} and \textit{process-message} procedures.

\begin{lemma}	\label{lem:exists-g}
For any non-closed node $s$ and for any optimal path $P$ from $I$ to $s$ which follows the restrictions of Lemma \ref{lem:restriction-on-P}, there exists an agent $\varphi$ which either has an open node $s'$ \textbf{or} has an incoming message containing $s'$, such that $s'$ is on $P$ and $g_{\varphi}(s')=g^*(s')$ .
\end{lemma}
\begin{proof}: Let $P=(I=n_0,n_1,\dots,n_k=s)$. If $I$ is in the open list of some agent $\varphi$ ($\varphi$ did not finish the algorithm's first iteration), let $s'=I$ and the lemma is trivially true since $g_{\varphi}(I)=g^*(I)=0$. Suppose $I$ is closed for all agents. Let $\Delta$ be the set of all nodes $n_i$ in $P$ that are closed by some agent $\varphi$, such that $g_\varphi(n_i)=g^*(n_i)$. $\Delta$ is not empty, since by assumption, $I\in \Delta$. Let $n_j$ be the element of $\Delta$ with the highest index, closed by agent $\varphi$. Clearly, $n_j\neq s$ since $s$ is non-closed. Let $a$ be the action causing the transition $n_j\rightarrow n_{j+1}$ in $P$. Therefore, $g^*(n_{j+1})=g_\varphi(n_j)+cost(a)$.

If $\varphi$ is the agent performing $a$, then $n_{j+1}$ is generated and moved to $\varphi$'s open list in lines \ref{line:succ}-\ref{line:move-to-open-list} of Algorithm \ref{alg:expand}, where $g_\varphi(n_{j+1})$ is assigned the value $g_\varphi(n_j)+cost(a)=g^*(n_{j+1})$ and the claim holds. 

Otherwise, $a$ is performed by agent $\varphi'\neq \varphi$. If $a$ is a public action, then all its preconditions hold in $n_j$, and therefore $n_j$ is sent to $\varphi'$ by $\varphi$ in line \ref{line:send} in Algorithm \ref{alg:expand}. If $a$ is a private action, by the definition of $P$, the next \textit{public} action $a'$ in $P$ is performed by $\varphi'$. Since private actions do not change the values of public variables, the public preconditions of $a'$ must hold in $n_j$, and therefore $n_j$ is sent to $\varphi'$ by $\varphi$ in line \ref{line:send} in Algorithm \ref{alg:expand}. Now, if the message containing $n_j$ has been processed by $\varphi'$, $n_j$ has been added to the open list of $\varphi'$ in Algorithm \ref{alg:process-message} and the claim holds since $g_{\varphi'}(n_j)=g_\varphi(n_j)=g^*(n_j)$. Otherwise, $\varphi'$ has an incoming (unprocessed) message containing $n_j$ and the claim holds since $g_\varphi(n_j)=g^*(n_j)$.
\end{proof}

\begin{corollary} \label{cor:exists-f}
Suppose $h_{\varphi}$ is admissible for every $\varphi\in \Phi$, and suppose the algorithm has not terminated. Then, for any optimal solution path $P$ which follows the restrictions of Lemma \ref{lem:restriction-on-P} from $I$ to any goal node $s^{\star}$, there exists an agent $\varphi_i$ which either has an open node $s$ \textbf{or} has an incoming message containing $s$, such that $s$ is on $P$ \textbf{and} $f_{\varphi_i}(s)\leq h^*(I)$.
\end{corollary}

\begin{proof}:
By Lemma \ref{lem:exists-g}, for every restricted optimal path $P$, there exists an agent $\varphi_i$ which either has an open node $s$ \textbf{or} has an incoming message containing $s$, such that $s$ is on $P$ and $g_{\varphi_i}(s)=g^*(s)$ . By the definition of $f$, and since $h_{\varphi_i}$ is admissible, we have in both cases:
\begin{displaymath} f_{\varphi_i}(s)= g_{\varphi_i}(s)+h_{\varphi_i}(s) = g^*(s)+h_{\varphi_i}(s)\end{displaymath} \begin{displaymath} \leq g^*(s)+h^*(s)=f^*(s) \end{displaymath}
But since $P$ is an optimal path, $f^*(n)=h^*(I)$, for all $n\in P$, which completes the proof. 
\end{proof}

Another lemma must be proved regarding the solution verification process. We assume global consistency of all heuristic functions, since all admissible heuristics can be made consistent by locally using the pathmax equation \cite{Mero84}, and by using the \textit{max} operator as in line \ref{line:update-h} of Algorithm \ref{alg:process-message} on heuristic values of different agents. This is required since $f_{\mathrm{lower-bound}}$ must be
non-decreasing. 

\begin{lemma}
\label{lem:termination}
Let $\varphi$ be an agent which either has an open node $s$ \textbf{or} has an incoming message containing $s$. Then, the solution verification procedure for state $s^*$ with $f(s^*)>f_{\varphi}(s)$ will return false.
\end{lemma}

\begin{proof}
Let $\varphi$ be an agent which either has an open node $s$ \textbf{or} has an incoming message containing $s$, such that $f_\varphi(s)<f(s^*)$ for some solution node $s^*$. The solution verification procedure for state $s^*$ verifies the stable property $p=``f(s^*)\leq f_{\mathrm{lower-bound}}$''. Since $f_{\mathrm{lower-bound}}$ represents the lowest $f$-value of any open or unprocessed state in the system, we have $f_{\mathrm{lower-bound}}\leq f_\varphi(s)<f(s^*)$, contradicting $p$. Relying on the correctness of the snapshot algorithm, this means that the solution verification procedure will return false, proving the claim.
%
\end{proof}

We can now prove the optimality of our algorithm.

\begin{theorem}
\masas\ terminates by finding a cost-optimal path to a goal node, if one exists.
\end{theorem}
\begin{proof}: We prove this theorem by assuming the contrary - the algorithm does not terminate by finding a cost-optimal path to a goal node. 3 cases are to be considered:
\begin{enumerate}
\item
\textit{The algorithm terminates at a non-goal node.} This contradicts the termination condition, since solution verification is initiated only when a goal state is expanded.
\item
\textit{The algorithm does not terminate.} Since we are dealing with a finite search space, let $\chi(\Pi)$ denote the number of possible \textit{non-goal} states. Since there are only a finite number of paths from $I$ to any node $s$ in the search space, $s$ can be \textit{reopened} a finite number of times. Let $\rho(\Pi)$ be the maximum number of times any \textit{non-goal} node $s$ can be reopened by any agent.
Let $t$ be the time point when all non-goal nodes $s$ with $f_{\varphi}(s)<h^*(I)$ have been closed forever by all agents $\varphi$. This $t$ exists since \textit{a}) we assume liveness of message passing and agent computations; \textit{b}) after at most $\chi(\Pi) \times \rho(\Pi)$ expansions of non-goal nodes by $\varphi$, all non-goal nodes of the search space must be closed forever by $\varphi$; and \textit{c}) no goal node $s^*$ with $f(s^*)<h^*(I)$ exists\footnote{This is needed since \textit{goal} node expansions are not bounded.}.

By Corollary \ref{cor:exists-f} and since an optimal path from $I$ to some goal state $s^*$ exists, some agent $\varphi$ expanded state $s^*$ at time $t'$,
such that $f_\varphi(s^*)\leq h^*(I)$. Since $s^*$ is an optimal solution, if $t'\geq t$, $f_{\mathrm{lower-bound}}\geq f_\varphi(s^*)$ at time $t'$. Therefore, $\varphi$'s verification procedure of $s^*$ will return true, and the algorithm terminates. 

Otherwise, $t'<t$. Let $\varphi'$ be the last agent to close a non-goal state $s$ with $f_{\varphi'}(s)<f_\varphi(s^*)$. $\varphi'$ has $s^*$ in its open list or as an incoming message. This is true because $s^*$ has been broad-casted to all agents by $\varphi$, and because  every time $s^*$ is closed by some agent (when it expands it), it is immediately broad-casted again, ending up in the agent's open list or in its message queue. Now, $\varphi'$ has no more open nodes with $f$-value lower than $s^*$, so it will eventually expand $s^*$, initiating the solution verification procedure which will return true, since $f_{\mathrm{lower-bound}}\geq f_\varphi(s^*)$. This contradicts the assumption of non-termination.

\item
\textit{The algorithm terminates at a goal node without achieving optimal cost.}
Suppose the algorithm terminates at some goal node $s$ with $f(s)>h^*(I)$. By Corollary \ref{cor:exists-f}, there existed just before termination an agent $\varphi$ having an open node $s'$, or having an incoming message containing $s'$, such that $s'$ is on an optimal path and $f_\varphi(s')\leq h^*(I)$. Therefore, by Lemma \ref{lem:termination}, the solution verification procedure for state $s$ will return false, contradicting the assumption that the algorithm terminated.
\end{enumerate}
This concludes the proof.
\end{proof}

\section{MA Planning Framework}
\label{framework}

One of the main goals of this work is to provide a general and scalable framework for solving the MA planning problem. We believe that such a framework will provide researchers with fertile ground for developing new search techniques and heuristics for MA planning, and extensions to richer planning formalisms.

We chose Fast Downward \cite{Helmert06} (FD) as the basis for our MA framework -- \textbf{MA-FD}.
FD is currently the leading framework for planning, both in the number of algorithms and heuristics that it provides, and in terms of performance -- winners of the past three international planning competitions were implemented on top of it. FD is also well documented and supported, so implementing and testing new ideas is relatively easy.

\mafd\ uses FD's translator and preprocessor, with minor changes to support the distribution of operators to agents. In addition to the PDDL files describing the domain and the problem instance, \mafd\ receives a file detailing the number of agents, their names, and their IP addresses. The agents do not have shared memory, and all information is relayed between agents using messages. Inter-agent communication is performed using the TCP\slash IP protocol, which enables running multiple \mafd\ agents as processes on multi-core systems, networked computers\slash robots, or even the cloud. \mafd\ is therefore fit to run as is on \textit{any} number of (networked) processors, in both its optimal and satisficing setting.

Both settings are currently implemented and available\footnote{The code is available at \url{https://github.com/raznis/dist-selfish-fd} .}, and since there is full flexibility regarding the heuristics used by agents, all heuristics available on FD are also available on \mafd. New heuristics are easily implementable, as in FD, and creating new search algorithms can also be done with minimal effort, since \mafd\ provides the ground-work (parsing, communication, etc.). 

\section{Empirical Results}
\label{empirical}

To evaluate \mafs\ in its non-optimal setting, we compare it to the state-of-the-art distributed planner \textsc{map-pop} \cite{TorrenoOS12}, and to the Planning-First algorithm~\cite{NissimBD10}. As noted in the Introduction, another available algorithm for distributed MA planning is via reduction to distributed CSPs using an off-the-shelf dis-CSP solver. We found this approach was incapable of solving even small planning problems, and therefore we omitted its results from the tables.
The problems used are benchmarks from the International Planning Competition (IPC)\nocite{ipc} where tasks can naturally be cast as MA problems. The \textit{Satellites} and \textit{Rovers} domains where motivated by real MA applications used by NASA. Satellites requires planning and scheduling observation tasks between multiple satellites, each equipped with different imaging tools. Rovers involves multiple rovers navigating a planetary surface, finding samples and communicating them back to a Lander. \textit{Logistics}, \textit{Transport} and \textit{Zenotravel} are transportation domains, where multiple vehicles transport packages to their destination. The Transport domain generalizes Logistics, adding a capacity to each vehicle (i.e., a limit on the number of packages it may carry) and different move action costs depending on road length. We consider problems from the Rovers and Satellites domains as loosely-coupled, i.e., problems where agents have many private actions (e.g., instrument warm-up and placement in Rovers, which does not affect other agents). On the other hand, we consider the transportation domains as tightly-coupled, having few private actions (only move actions in Logistics) and many public actions (all load/unload actions).

For each planning problem, we ran \mafs, using eager best-first search and an alternation open list with one queue for each of the two heuristic functions \emph{ff} \cite{Hoffmann200114253} and the \emph{context-enhanced additive heuristic} \cite{HelmertG08}. Table \ref{exp:sat} depicts results of \mafs, \textsc{map-pop} and Planning-First, on all IPC domains supported by \textsc{map-pop}. We compare the algorithms across three categories -- 1) solution qualit,y which reports the cost of the outputted plan, 2) running time, and 3) the number of messages sent during the planning process. Experiments were run on a AMD Phenom 9550 2.2GHZ processor, time limit was set at 60 minutes, and memory usage was limited to 4GB. An ``X'' signifies that the problem was not solved within the time-limit, or exceeded memory constraints.

\begin{table}[t]\footnotesize
\centering \setlength{\tabcolsep}{5.7pt}
\caption{Comparison of MA greedy best-first search, \textsc{map-pop} and Planning-First. Solution cost, running time (in sec.) and the number of sent messages are shown. \textit{``X''} denotes problems which weren't solved after one hour, or in which the 4GB memory limit was exceeded.}
\label{exp:sat}
\begin{tabular}{|l|c||r|r|r||r|r|r||r|r|r|} \hline
 &\#  & \multicolumn{3}{|c||}{Solution cost} & \multicolumn{3}{|c||}{Runtime} & \multicolumn{3}{|c|}{Messages}\\
problem	&	agents	&	\mafs	&	\textsc{map-pop}	&	\textsc{p-f}	&	\mafs	&	\textsc{map-pop}	&	\textsc{p-f}	&	\mafs	&	\textsc{map-pop}	&	\textsc{p-f}	\\	\hline
Logistics4-0	&	3	&	20	&	20	&	X	&	\textbf{0.05}	&	20.9	&	X	&	\textbf{340}	&	375	&	X	\\	
Logistics5-0	&	3	&	27	&	27	&	X	&	\textbf{0.1}	&	90.4	&	X	&	\textbf{450}	&	1565	&	X	\\	
Logistics6-0	&	3	&	25	&	25	&	X	&	\textbf{0.06}	&	60.6	&	X	&	\textbf{470}	&	1050	&	X	\\	
Logistics7-0	&	4	&	36	&	37	&	X	&	\textbf{0.2}	&	233.3	&	X	&	\textbf{2911}	&	4898	&	X	\\	
Logistics8-0	&	4	&	31	&	31	&	X	&	\textbf{0.16}	&	261	&	X	&	\textbf{940}	&	4412	&	X	\\	
Logistics9-0	&	4	&	36	&	36	&	X	&	\textbf{1.02}	&	193.3	&	X	&	\textbf{2970}	&	3168	&	X	\\	
Logistics10-0	&	5	&	\textbf{45}	&	51	&	X	&	\textbf{0.43}	&	471	&	X	&	\textbf{2097}	&	14738	&	X	\\	
Logistics11-0	&	5	&	\textbf{54}	&	X	&	X	&	\textbf{2.7}	&	X	&	X	&	\textbf{14933}	&	X	&	X	\\	
Logistics12-0	&	5	&	\textbf{44}	&	45	&	X	&	\textbf{1.3}	&	1687	&	X	&	\textbf{4230}	&	28932	&	X	\\	
Logistics13-0	&	7	&	\textbf{87}	&	X	&	X	&	\textbf{0.9}	&	X	&	X	&	\textbf{5140}	&	X	&	X	\\	
Logistics14-0	&	7	&	\textbf{68}	&	X	&	X	&	\textbf{0.67}	&	X	&	X	&	\textbf{2971}	&	X	&	X	\\	
Logistics15-0	&	7	&	\textbf{95}	&	X	&	X	&	\textbf{0.74}	&	X	&	X	&	\textbf{6194}	&	X	&	X	\\	\hline
Rovers5	&	2	&	\textbf{22}	&	24	&	24	&	\textbf{0.13}	&	18.7	&	22.4	&	\textbf{84}	&	323	&	590	\\	
Rovers6	&	2	&	\textbf{37}	&	39	&	X	&	\textbf{0.07}	&	18.2	&	X	&	\textbf{27}	&	313	&	X	\\	
Rovers7	&	3	&	18	&	18	&	X	&	\textbf{0.07}	&	44.1	&	X	&	\textbf{225}	&	490	&	X	\\	
Rovers8	&	4	&	\textbf{26}	&	27	&	X	&	\textbf{0.2}	&	744	&	X	&	\textbf{937}	&	12102	&	X	\\	
Rovers9	&	4	&	38	&	\textbf{36}	&	X	&	\textbf{0.82}	&	222	&	X	&	\textbf{380}	&	4467	&	X	\\	
Rovers10	&	4	&	\textbf{38}	&	X	&	X	&	\textbf{0.41}	&	X	&	X	&	\textbf{271}	&	X	&	X	\\	
Rovers11	&	4	&	37	&	\textbf{34}	&	X	&	\textbf{0.34}	&	132.5	&	X	&	\textbf{299}	&	2286	&	X	\\	
Rovers12	&	4	&	21	&	\textbf{20}	&	X	&	\textbf{0.09}	&	34.4	&	X	&	435	&	\textbf{410}	&	X	\\	
Rovers13	&	4	&	\textbf{49}	&	X	&	X	&	\textbf{0.15}	&	X	&	X	&	\textbf{472}	&	X	&	X	\\	
Rovers14	&	4	&	\textbf{31}	&	35	&	X	&	\textbf{0.42}	&	443.8	&	X	&	\textbf{310}	&	7295	&	X	\\	
Rovers15	&	4	&	46	&	\textbf{44}	&	X	&	\textbf{0.33}	&	164	&	X	&	\textbf{252}	&	2625	&	X	\\	
Rovers17	&	6	&	\textbf{52}	&	X	&	X	&	\textbf{0.57}	&	X	&	X	&	\textbf{628}	&	X	&	X	\\	\hline
Satellites3	&	2	&	11	&	11	&	11	&	\textbf{0.01}	&	4.5	&	6.8	&	\textbf{7}	&	78	&	104	\\	
Satellites4	&	2	&	\textbf{17}	&	20	&	20	&	\textbf{0.17}	&	6.4	&	35.2	&	\textbf{36}	&	109	&	144	\\	
Satellites5	&	3	&	16	&	\textbf{15}	&	X	&	\textbf{0.15}	&	15.4	&	X	&	\textbf{78}	&	250	&	X	\\	
Satellites6	&	3	&	20	&	20	&	X	&	\textbf{0.02}	&	12.2	&	X	&	\textbf{30}	&	323	&	X	\\	
Satellites7	&	4	&	22	&	22	&	X	&	\textbf{0.23}	&	28.8	&	X	&	\textbf{248}	&	543	&	X	\\	
Satellites8	&	4	&	26	&	26	&	X	&	\textbf{0.21}	&	40.7	&	X	&	\textbf{133}	&	678	&	X	\\	
Satellites9	&	5	&	30	&	\textbf{29}	&	X	&	\textbf{0.35}	&	93.3	&	X	&	\textbf{397}	&	1431	&	X	\\	
Satellites10	&	5	&	30	&	\textbf{29}	&	X	&	\textbf{0.41}	&	65.9	&	X	&	\textbf{355}	&	942	&	X	\\	
Satellites11	&	5	&	31	&	31	&	X	&	\textbf{0.69}	&	51	&	X	&	\textbf{514}	&	904	&	X	\\	
Satellites12	&	5	&	\textbf{43}	&	49	&	X	&	\textbf{1.1}	&	76.9	&	X	&	\textbf{390}	&	1240	&	X	\\	
Satellites13	&	5	&	61	&	X	&	X	&	\textbf{0.88}	&	X	&	X	&	\textbf{639}	&	X	&	X	\\	
Satellites14	&	6	&	44	&	\textbf{43}	&	X	&	\textbf{1.8}	&	123.4	&	X	&	\textbf{721}	&	1781	&	X	\\	
Satellites15	&	8	&	\textbf{63}	&	X	&	X	&	\textbf{3.9}	&	X	&	X	&	\textbf{1507}	&	X	&	X	\\	
Satellites16	&	10	&	56	&	56	&	X	&	\textbf{6.6}	&	481.2	&	X	&	\textbf{2279}	&	4942	&	X	\\	
Satellites17	&	12	&	49	&	49	&	X	&	\textbf{6.7}	&	2681	&	X	&	\textbf{2172}	&	26288	&	X	\\	\hline
\end{tabular}
\end{table}

\begin{table}[t]\footnotesize
\centering 
\caption{Comparison of centralized \astar\ and \masas\ running on multiple processors. Running time (in sec.), average initial state $h$-values and \masas's efficiency values w.r.t.\ \astar\ are shown.} 
\label{exp:distributed}
\begin{tabular}{|l|c||r|r|r||r|r||r|r|} \hline
 & &  \multicolumn{3}{|c||}{Time} & \multicolumn{2}{|c||}{Expansions} & \multicolumn{2}{|c|}{init-h}  \\ 
problem	&	agents	&	\astar	&	\masas	&	Efficiency	&	\astar	&	\masas	&	\astar	&	\masas	\\	\hline
Logistics4-0	&	3	&	0.07	&	0.03	&	0.78	&	21	&	2496	&	20	&	17	\\	
Logistics5-0	&	3	&	0.16	&	0.13	&	0.41	&	28	&	11020	&	27	&	23	\\	
Logistics6-0	&	3	&	0.29	&	0.15	&	0.64	&	547	&	9722	&	24	&	22	\\	
Logistics7-0	&	4	&	1.42	&	8.26	&	0.04	&	29216	&	520122	&	33	&	29	\\	
Logistics8-0	&	4	&	1.28	&	2.89	&	0.11	&	16771	&	157184	&	27	&	24	\\	
Logistics9-0	&	4	&	2.17	&	11.6	&	0.05	&	43283	&	687572	&	33	&	29	\\	
Logistics10-0	&	5	&	132	&	X	&	0.00	&	4560551	&	X	&	37	&	34	\\	
Logistics11-0	&	5	&	180	&	X	&	0.00	&	5713287	&	X	&	41	&	38	\\	\hline
Rovers3	&	2	&	0.2	&	0.11	&	0.91	&	12	&	86	&	11	&	9	\\	
Rovers4	&	2	&	0.07	&	0.04	&	0.88	&	9	&	121	&	8	&	7	\\	
Rovers5	&	2	&	8.24	&	4.4	&	0.94	&	213079	&	307995	&	12	&	11	\\	
Rovers6	&	2	&	X	&	301	&	\textbf{$\infty$}	&	X	&	\textbf{18274357}	&	24	&	21	\\	
Rovers7	&	3	&	32.4	&	6.82	&	\textbf{1.58}	&	1172964	&	\textbf{676750}	&	9	&	7	\\	
Rovers12	&	4	&	190.8	&	27.6	&	\textbf{1.73}	&	4963979	&	\textbf{2591428}	&	11	&	7	\\	\hline
Satellites3	&	2	&	0.3	&	0.29	&	0.52	&	12	&	1498	&	11	&	5	\\	
Satellites4	&	2	&	0.6	&	0.4	&	0.75	&	18	&	5045	&	17	&	11	\\	
Satellites5	&	3	&	16.83	&	3.9	&	\textbf{1.44}	&	236647	&	\textbf{42557}	&	10	&	7	\\	
Satellites6	&	3	&	1.93	&	0.92	&	0.70	&	1382	&	81731	&	17	&	12	\\	
Satellites7	&	4	&	X	&	18.26	&	\textbf{$\infty$}	&	X	&	\textbf{1303910}	&	10	&	9	\\	\hline
Transport1	&	2	&	0.02	&	0.08	&	0.13	&	6	&	285	&	54	&	4	\\	
Transport2	&	2	&	0.17	&	0.21	&	0.40	&	242	&	1628	&	79	&	4	\\	
Transport3	&	2	&	1.35	&	20.64	&	0.03	&	69500	&	546569	&	34	&	4	\\	
Transport4	&	2	&	54.6	&	335.15	&	0.08	&	3527592	&	4397124	&	10	&	10	\\	
Transport5	&	2	&	X	&	X	&	N/A	&	X	&	X	&	12	&	10	\\	
Transport6	&	2	&	X	&	X	&	N/A	&	X	&	X	&	6	&	5	\\	\hline
Zenotravel3	&	2	&	0.33	&	0.42	&	0.39	&	7	&	516	&	6	&	5	\\	
Zenotravel4	&	2	&	0.32	&	0.43	&	0.37	&	9	&	762	&	7	&	6	\\	
Zenotravel5	&	2	&	0.3	&	0.42	&	0.36	&	14	&	577	&	10	&	8	\\	
Zenotravel6	&	2	&	0.47	&	0.61	&	0.39	&	270	&	1280	&	10	&	8	\\	
Zenotravel7	&	2	&	0.55	&	0.8	&	0.34	&	621	&	5681	&	11	&	9	\\	
Zenotravel8	&	3	&	1.22	&	1.71	&	0.24	&	136	&	5461	&	9	&	7	\\	
Zenotravel9	&	3	&	31.7	&	315	&	0.03	&	775340	&	6745088	&	16	&	14	\\	
Zenotravel10	&	3	&	9.3	&	338	&	0.01	&	116872	&	4416461	&	20	&	17	\\	
Zenotravel11	&	3	&	2.99	&	11.7	&	0.09	&	20157	&	200868	&	11	&	9	\\	
Zenotravel12	&	3	&	X	&	X	&	N/A	&	X	&	X	&	16	&	14	\\	\hline
\end{tabular}
\end{table}

It is clear that \mafs\ overwhelmingly dominates both \textsc{map-pop} and Planning-First (denoted \textsc{p-f}), with respect to running time and communication, solving \emph{all} problems faster while sending less messages. All problems were solved at least 70X faster than \textsc{map-pop}, with several Logistics and Rovers problems being solved over 1000X faster, and the largest Satellites instance being solved over 400X faster. The low communication complexity of \mafs\ is important, since in distributed systems message passing could be more costly and time-consuming than local computation. Moreover, as messages in \mafs\ are essentially a state description, message size is linear in the number of propositions. Although for some problems \textsc{map-pop} finds lower-cost solutions, in most cases \mafs\ outputs better solution quality. We believe that when \mafs\ finds lower quality solutions, this is mostly because message-passing takes longer than local computation -- a subset of agents that have the ability to achieve the goal ``on their own'', will do so before being made aware of other, less costly solutions including other agents. One possible way of improving solution quality further would be using anytime search methods, which improve solution quality over time. 

To evaluate \masas\ with respect to centralized optimal search (\astar), we ran both algorithms using the state-of-the-art Merge\&Shrink heuristic\footnote{We used exact bisimulation with abstraction size limit $10K$ (DFP-bop) as the shrink strategy of Merge\&Shrink \cite{NissimHH11}.} \cite{HelmertHH07}. Both configurations were run on the same multi-core machine, where for \masas, each agent was allocated a single processor, and \astar\ was run on a single processor. Time limit was set at 30 minutes, and memory usage was limited to 4GB, regardless of the number of cores used. Table \ref{exp:distributed} depicts the runtime, efficiency (speedup divided by the number of processors), number of expanded nodes and the average of the agents' initial state $h$-values. When comparing \masas\ to centralized \astar, our intuition is that efficiency will be low, due to the inaccuracy of the agents' heuristic estimates, and the overhead incurred by communication. In fact, the local estimates of the agents are much less accurate than those of the global heuristic, as is apparent from the lower average $h$ values of the initial state, given as approximate measures of heuristic quality.
In the tightly-coupled domains -- Logistics, Transport and Zenotravel, we do notice very low efficiency values, mostly due to the large number of public actions, which result in many messages being passed between agents. However, in the more loosely-coupled domains Satellites and Rovers, \masas\ exhibits nearly linear and super-linear speedup, solving 2 problems not solved by centralized \astar. We elaborate on this important issue in the next section.

\section{Partition-Based Path Pruning}
\label{sec:partition}
The empirical results presented in Table \ref{exp:distributed} raise an interesting question: \emph{how does \masas\ achieve $>1$ efficiency in weakly coupled environments?} It is known that when using a consistent heuristic, \astar\ is optimal in the number of nodes it expands to recognize an optimal solution. In principle, it appears that \masas\ should expand at least the same search tree, so it is not clear, a-priori, why we reach super-linear speedup when comparing to \astar. The main reason for this is \masas's inherent exploitation of symmetry, resulting in the pruning of effect-equivalent paths. 

Symmetry exploitation utilizes the notion of public and private actions. As we noted in Corollary \ref{lem:restriction-on-P}, the existence of private actions implies the existence of multiple effect-equivalent permutations of certain action sequences. \astar\ does not recognize or exploit this fact, and \mafs\ does. Specifically, imagine that agent $\varphi_i$ just generated state $s$ using one of its public actions, and $s$ satisfies the preconditions of some public action $a$ of agent $\varphi_j$. Agent $\varphi_i$ will eventually send $s$ to agent $\varphi_j$, and the latter will eventually apply $a$ to it. Now, imagine that agent $\varphi_i$ has a private action $a'$ applicable at state $s$, resulting in the state $s'=a'(s)$. Because $a'$ is private to $\varphi_i$, from the fact that $a$ is applicable at $s$ we deduce that $a$ is applicable at $s'$ as well. Hence, \astar\ would apply $a$ at $s'$. However, in \mafs, agent $\varphi_j$ would not apply $a$ at $s'$ because it will not receive $s'$ from agent $\varphi_i$. Thus, \mafs\ does not explore all possible action sequences. This fact can also be clearly seen in the example given in Figure \ref{fig:example} -- The reachable search space in this example has 31 states, while the number of reachable states using \mafs\ is only 16. 

Since \mafs's inherent pruning of action sequences requires only a partitioning of the actions, it does not pertain only to MA systems, but to any factored system having internal operators. Since the only difference between a \mastrips\ planning problem and a \strips\ one is the fact that actions are partitioned between agents, why not re-factor the centralized problem into an ``artificial'' MA one? By mapping all actions into disjoint sets such that $\bigcup_i^k A_i = A$, each representing an ``agent'', we can now distinguish between private and public operators. Given this distinction, the pruning rule used is simple:
\begin{quote}
{\bf Partition-Based (PB) Pruning Rule:} \emph{Following a private action $a\in A_i$, prune all actions not in $A_i$}.
\end{quote}

The fact that this pruning rule is optimality-preserving (i.e., does not prune \emph{all} optimal solutions) follows immediately from Corollary \ref{lem:restriction-on-P}, as if there exists an optimal solution $\pi^\star$, it can be permuted into a legal, optimal plan which is not pruned. This, however, is not enough to maintain the optimality of \astar\ search. We now present a slight modification of the \astar~algorithm, which allows the application of optimality preserving pruning methods (such as PB-pruning) for the purpose of optimal planning.

\subsection{Path Pruning A*}
\label{sec:theory}
The path pruning \astar, (denoted \ppastar), is a search algorithm which receives a planning problem $\Pi$ and a pruning method $\prune$ as input, and produces a plan $\pi$, which is guaranteed to be optimal provided that $\prune$ respects the following properties: (i) $\prune$ is optimality preserving, and (ii) $\prune$ prunes only according to the last action. It is easy to see, for example, that PB pruning respects the second condition, since it fires only according to the last action. 

\subsubsection{\ppastar~versus \astar}
\ppastar~is identical to \astar~except for the following three changes. First, a different data-type is used for recording an open node. In \ppastar, an open list node is a pair $(A,s)$, where $s$ is the state and $A$ is a set of actions, recording various possible ways to reach $s$ from a previous state. Second, node expansion is subject to the pruning rules of method $\prune$. Namely, \ppastar~executes an applicable action $a'$ in state $(A,s)$ only if there is at least one action $a \in A$ s.t.~the execution of $a'$ is allowed after $a$ under $\prune$'s pruning rules. Third, duplicate states are handled differently. In \astar, when a state $s$ which is already open is reached by another search path, the open list node is updated with the action of the lower $g$ value, and in case of a tie -- drops the competing path. In contrast, ties in \ppastar~are handled by preserving the last actions which led to $s$ in each of the paths. Hence, if action $a$ led to an open state $s$ via a path of cost $g$, and if the existing open list node $(A,s)$  has the same $g$ value, then the node is updated to $(A\cup \{a\},s)$, thus all actions leading to $s$ with path cost $g$ are saved. Tie breaking also affects the criterion under which closed nodes are reopened. In \astar, nodes are reopened only when reached via paths of lower $g$ values. In \ppastar, if an action $a$ leading to state $s$ of some closed node $(A,s)$ is not contained in $A$, and if the $g$ values are equal, then the node reopens as $(\{A \cup\{a\}\},s)$. However, when the node is expanded, only actions that are now allowed by $\prune$ and were previously pruned, are executed. 
We now move to prove the correctness of \ppastar.

\subsubsection{Proof of Correctness and Optimality}
The next lemma refers to \ppastar, and assumes $\prune$ to be an optimality preserving pruning method, which prunes according to the last action. We say that node $(A,s)$ is \emph{optimal on path $P$}, if $A$ contains an action $a$ which leads to state $s$ on path $P$, and $g(s)=g^*(s)$. The notation $s \prec_P s'$ denotes the fact that state $s$ precedes state $s'$ in optimal path $P$.
\begin{lemma}	
\label{lem:ppastar-exists-g}
In \ppastar, for any non-closed state $s_k$ and for any optimal non-$\prune$-pruned path $P$ from $I$ to $s_k$, there exists an open list node $(A',s')$ which is optimal on $P$.
\end{lemma}
\begin{proof}
Let $P$ be an optimal non-$\prune$-pruned path from $I$ to $s_k$. If I is in the open list, let $s'=I$ and the lemma is trivially true since $g(I)=g^*(I)=0$. Suppose $I$ is closed. Let $\Delta$ be the set of all nodes $(A_i,s_i)$ optimal on $P$, that were closed.  $\Delta$ is not empty, since by assumption, $I$ is in $\Delta$. Let the nodes in $\Delta$ be ordered such that $s_i \prec_P s_j$ for $i<j$, and  let $j$ be the highest index of any $s_i$ in $\Delta$.

Since the closed node $(A_j,s_j)$ has an optimal $g$ value, it had been expanded prior to closing. From the properties of \ppastar, it follows that the expansion of $(A_j,s_j)$, which is optimal on $P$, is followed with an attempt to generate a node $(A_{j+1},s_{j+1})$ which is optimal on $P$ as well. Generation of $(A_{j+1},s_{j+1})$ must be allowed, since under the highest index assumption there can be no closed node containing $s$ which is optimal on $P$. Naturally, $s_j \prec_P s_{j+1}$. 

At this point, we note that actions in $A_{j+1}$ cannot be removed by any competing path from $I$ to $s_{j+1}$, since $(A_{j+1},s_{j+1})$ has an optimal $g$ value. It is possible, though, that additional actions leading to $s_{j+1}$ are added to the node. The updated node can be represented by $(A'_{j+1} \supseteq A_{j+1},s_{j+1})$, and the property of optimality on $P$ holds. Additionally, node $(A'_{j+1},s_{j+1})$ cannot be closed after its generation, since again, this contradicts the highest index property. Hence, there exists an open list node $(A',s')$ which is optimal on $P$. This concludes the proof.
\end{proof}
\begin{corollary} \label{ma-completeness}
If $h$ is admissible and $\prune$ is optimality-preserving, \ppastar~using $\prune$ is optimal. 
\end{corollary}
\begin{proof} 
This follows directly from Lemma \ref{lem:ppastar-exists-g}, the optimality preserving property of $\prune$ and the properties of \ppastar, which allow every optimal, non-$\prune$-pruned path to be generated.
\end{proof}

\subsection{Empirical Analysis of PB-Pruning}
\begin{table}[t]\footnotesize
\centering 
\caption{Comparison of centralized \astar\ with and without partition-based pruning, and \masas\ running on multiple processors. Running time, number of expanded nodes, and \masas's efficiency w.r.t.\ both centralized configurations are shown.} 
\label{exp:pruning}
\begin{tabular}{|l|c||r|r|r||r|r||r|r|r|} \hline
 & &  \multicolumn{3}{|c||}{Time} &  \multicolumn{2}{|c||}{Efficiency} & \multicolumn{3}{|c|}{Expansions}   \\ 
problem	&	agents	&	\astar	&	$\textsc{a}^\star_{pb}$	&	\masas	&	\astar	&	$\textsc{a}^\star_{pb}$	&	\astar	&	$\textsc{a}^\star_{pb}$	&	\masas	\\	\hline
Logistics4-0	&	3	&	0.07	&	0.06	&	0.03	&	0.78	&	0.67	&	21	&	21	&	2496	\\	
Logistics5-0	&	3	&	0.16	&	0.17	&	0.13	&	0.41	&	0.44	&	28	&	28	&	11020	\\	
Logistics6-0	&	3	&	0.29	&	0.29	&	0.15	&	0.64	&	0.64	&	547	&	527	&	9722	\\	
Logistics7-0	&	4	&	1.42	&	1.07	&	8.26	&	0.04	&	0.03	&	29216	&	22425	&	520122	\\	
Logistics8-0	&	4	&	1.28	&	1.07	&	2.89	&	0.11	&	0.09	&	16771	&	11750	&	157184	\\	
Logistics9-0	&	4	&	2.17	&	1.59	&	11.6	&	0.05	&	0.03	&	43283	&	29953	&	687572	\\	
Logistics10-0	&	5	&	132	&	37	&	X	&	0	&	0	&	4560551	&	2132416	&	X	\\	
Logistics11-0	&	5	&	180	&	57.4	&	X	&	0	&	0	&	5713287	&	2980725	&	X	\\	\hline
Rovers3	&	2	&	0.2	&	0.2	&	0.11	&	0.91	&	0.91	&	12	&	12	&	86	\\	
Rovers4	&	2	&	0.07	&	0.06	&	0.04	&	0.88	&	0.75	&	9	&	9	&	121	\\	
Rovers5	&	2	&	8.24	&	3.07	&	4.4	&	0.94	&	0.35	&	213079	&	62672	&	307995	\\	
Rovers6	&	2	&	X	&	164.9	&	301	&	$\infty$	&	0.27	&	X	&	8107327	&	18274357	\\	
Rovers7	&	3	&	32.4	&	5.22	&	6.82	&	1.58	&	0.26	&	1172964	&	235537	&	676750	\\	
Rovers12	&	4	&	190.8	&	10.1	&	27.6	&	1.73	&	0.09	&	4963979	&	391372	&	2591428	\\	\hline
Satellites3	&	2	&	0.3	&	0.29	&	0.29	&	0.52	&	0.50	&	12	&	12	&	1498	\\	
Satellites4	&	2	&	0.6	&	0.58	&	0.4	&	0.75	&	0.73	&	18	&	18	&	5045	\\	
Satellites5	&	3	&	16.83	&	3.15	&	3.9	&	1.44	&	0.27	&	236647	&	23503	&	42557	\\	
Satellites6	&	3	&	1.93	&	1.84	&	0.92	&	0.70	&	0.67	&	1382	&	385	&	81731	\\	
Satellites7	&	4	&	X	&	26.6	&	18.26	&	$\infty$	&	0.36	&	X	&	846394	&	1303910	\\	\hline
Transport1	&	2	&	0.02	&	0.01	&	0.08	&	0.13	&	0.06	&	6	&	6	&	285	\\	
Transport2	&	2	&	0.17	&	0.16	&	0.21	&	0.40	&	0.38	&	242	&	225	&	1628	\\	
Transport3	&	2	&	1.35	&	0.9	&	20.64	&	0.03	&	0.02	&	69500	&	49744	&	546569	\\	
Transport4	&	2	&	54.6	&	29.8	&	335.15	&	0.08	&	0.04	&	3527592	&	2589496	&	4397124	\\	\hline
Zenotravel3	&	2	&	0.33	&	0.34	&	0.42	&	0.39	&	0.40	&	7	&	7	&	516	\\	
Zenotravel4	&	2	&	0.32	&	0.33	&	0.43	&	0.37	&	0.38	&	9	&	9	&	762	\\	
Zenotravel5	&	2	&	0.3	&	0.3	&	0.42	&	0.36	&	0.36	&	14	&	14	&	577	\\	
Zenotravel6	&	2	&	0.47	&	0.47	&	0.61	&	0.39	&	0.39	&	270	&	220	&	1280	\\	
Zenotravel7	&	2	&	0.55	&	0.56	&	0.8	&	0.34	&	0.35	&	621	&	433	&	5681	\\	
Zenotravel8	&	3	&	1.22	&	1.21	&	1.71	&	0.24	&	0.24	&	136	&	126	&	5461	\\	
Zenotravel9	&	3	&	31.7	&	12.88	&	315	&	0.03	&	0.01	&	775340	&	474180	&	6745088	\\	
Zenotravel10	&	3	&	9.3	&	6.51	&	338	&	0.01	&	0.01	&	116872	&	104340	&	4416461	\\	
Zenotravel11	&	3	&	2.99	&	2.02	&	11.7	&	0.09	&	0.06	&	20157	&	11565	&	200868	\\	
Zenotravel12	&	3	&	X	&	82.95	&	X	&	N/A	&	0	&	X	&	2406708	&	X	\\	\hline
\end{tabular}
\end{table}

We set out to check the effect of \masas's inherent exploitation of symmetry on its efficiency compared to \astar. The hypothesis that this is \masas's main advantage over \astar\ is well supported by the results in Table \ref{exp:pruning}, which shows a comparison of \masas\ and centralized \astar\ using PB pruning. Here, we see that in all problems where \masas\ achieves superlinear speedup w.r.t.\ \astar, applying partition-based pruning where \textit{partition=agent} reduces runtime and expansions dramatically. In all cases, \masas's efficiency w.r.t.\ \astar\ using PB pruning is sublinear. This is, of course, also due to the fact that \masas\ solves a more difficult problem -- having incomplete information has a negative effect on the quality of heuristics computed by the agents.

Finally, we note that although MA structure is evident in some benchmark planning domains (e.g.\ Logistics, Rovers, Satellites, Zenotravel etc.), in general there isn't always an obvious way of decomposing the problem. In work further exploring PB pruning \cite{NissimAB12}, we describe an automated method for decomposing a general planning problem, making PB pruning applicable in the general setting.

\section{Discussion}
\label{discussion}

We presented a formulation of heuristic forward search for distributed systems that respects the natural distributed structure of the
system. \mafs\ dominates the state-of-the-art w.r.t.\ runtime and communication, as well as solution quality in most cases. In this class of
privacy preserving algorithms, \masas, is the first cost-optimal distributed planning algorithm, and it is competitive with its centralized counterpart, despite having partial information. Our work raises a number of research challenges and opportunities, which we now discuss.

\subsection{Accurate Heuristics Given Incomplete Information}
The empirical results presented lead us to what is perhaps the greatest practical challenge suggested by \mafs\ and \masas\ --  computing an accurate heuristic in a distributed (privacy-preserving) system. In some domains, the existence of private information that is not shared leads to serious deterioration in the quality of the heuristic function, greatly increasing the number of nodes expanded, and\slash or affecting solution quality. We believe that there are techniques that can be used to alleviate this problem. As a simple example, consider a public action $a_{\mathrm{pub}}$ that can be applied only after a private action $a_{\mathrm{priv}}$. For example, in the rover domain, a \textit{send} message can only be applied after various private actions required to collect data are executed. If the cost of $a_{\mathrm{pub}}$ known to other agents would reflect the cost of $a_{\mathrm{priv}}$ as well, the heuristic estimates would be more accurate. Another possibility for improving heuristic estimates is using an additive heuristic.  In that case, rather than taking the maximum of the agent's own heuristic estimate and the estimate of the sending agent, the two could be added. To maintain admissibility, this would require using something like cost partitioning \cite{KatzD08}. 
One obvious way of doing this would be to give each agent the full cost of its actions and zero cost for other actions.  The problem with this approach is that initially, when the state is generated and the only estimate available is that of the generating agent, this estimate is very inaccurate, since it assigns 0 to all other actions. In fact, the agent will be inclined to prefer actions performed by other agents, as they appear very cheap, and we see especially poor results in domains where different agents can achieve the same goal, as in the Rovers domain, resulting in estimates of 0 for many non-goal states. Therefore, how to effectively compute accurate heuristics in the distributed setting remains an open challenge.

\subsection{Secure Multi-Party Computation}
Secure Multi-Party Computation \cite{Yao82b,Yao86} is a subfield of Cryptography which relates closely to distributed planning, as well as to distributed problem solving in general. The goal of methods for secure multi-party computation is to enable multiple agents to compute a function over their inputs, while keeping these inputs private. More specifically, agents $\varphi_1,\dots,\varphi_n$, having \emph{private} data $x_1,\dots,x_n$, would like to jointly compute some function $f(x_1,\dots,x_n)$, without revealing any information about their private information, other than what can be reasonably deduced from the value of $f(x_1,\dots,x_n)$. 

While in principle, it appears that these techniques can be extended to our setting of distributed planning, their complexity quickly becomes unmanageable. For example, a common approach for secure multiparty computation uses cryptographic circuits. When solving the shortest path problem (e.g., network routing, Gupta et al., 2012\nocite{GuptaSPSSFRS12}), the size of the circuits created is polynomial in the size of the graph. In our setting the function $f$ computes a shortest path in the implicit graph induced by the descriptions of the agents' actions. As this graph is exponential in the problem description size, it quickly becomes infeasible to construct these circuits given time and memory limitations. While it is true that planning is NP-hard and forward search algorithms do, in general, require exponential time/memory, the purpose of heuristic search is to reduce the search space and to solve large problems in low-polynomial time. Requiring the construction of exponential-sized circuits {\em a-priori}  contradicts the goal of efficiency and feasibility. Another difference between our model and the ones used for secure multiparty computation, is that these methods assume that some ($\geq 1$) of the agents are honest, and the other agents are adversaries which are determined to uncover the private information. 
In distributed planning, the distinction between honest agents and adversaries is not as clear-cut. Despite faithfully participating in the distributed protocol, all agents might benefit from discovering other agents' private information (e.g., competing companies or contractors), and therefore can all be viewed as adversaries where privacy is concerned. Therefore, the assumptions usually made for secure multiparty computation regarding the limited number of adversaries do not fit our models as well.

\subsection{Privacy}
\label{sec:discussion_privacy}
Work in distributed CSPs \cite{YokooSH02,SilaghiM04} identified that although a key motivation for distributed computation is preservation of agent privacy, some private information may leak during the search process. For example, in DisCSP each agent has a single variable, and there exist both binary and unary constraints. Binary constraints are public since more than one agent knows of their existence, while unary constraints are considered private information. In meeting scheduling, an agent has a single variable whose values are possible meeting time slots. A binary constraint could be an equality constraint between the values of two variables belonging to different agents, while a unary constraint represents slots in which the agent cannot hold meetings. During search, whenever an agent sends some assignment of its variable to other agents, they can deduce that that value has no unary constraint forbidding it. If this value does not end up being assigned in the solution, the agent revealed some private information that could not have been deduced from only viewing the solution. In the field of DisCSPs, there has been work focusing of how to \emph{measure} this privacy loss \cite{FranzinRFW04,MaheswaranPBVT06}, as well as work on analyzing how much information specific algorithms lose \cite{GreenstadtPT06}. More recently, further work has emerged on how to alter existing DisCSP algorithms to handle stricter privacy demands \cite{GreenstadtGS07,LeauteF09}.

In our model of distributed planning, things are a bit different. To consider privacy loss during search, first we must examine what type of information could leak during  distributed search. Our model considers the private preconditions of public actions, private actions, and private
action costs as private information which the agents do not want to disclose. 

We begin by discussing the cost of private actions. When running \masas, messages sent by the agents contain $g$-values, or the currently minimal cost of arriving at a state. Given this information throughout the search procedure, the agents can deduce an upper bound on the minimal cost of applying public action $a$, given a \emph{public} state $s$. Consider an example of a system consisting of two agents $\varphi_{1,2}$. During the search procedure, $\varphi_1$ sends public state $s$ to $\varphi_2$ multiple times, each with different private states, which are indistinguishable to $\varphi_2$, by applying methods discussed in Section \ref{sec:privacy}. Upon receiving $s$, $\varphi_2$ continues searching until applying public action $a$, and then sending the resulting state $s'$ back to $\varphi_1$, which can now compute $g(s')-g(s)$, or the total (including private) cost of applying $a$. If $\varphi_1$ minimizes this value for every $s'$, it can now deduce an upper bound on the minimal cost of applying $a$ given public state $s$. 

Another possible leak of information can be the existence of private actions or private preconditions of public actions. These affect whether or not a public action can be applied at certain states. When running \mafs, the first bit of information which can easily be deduced is whether a public action is applicable in some reachable state. Clearly, if an agent sends a state for which the creating operator is public action $a$, then other agents now know that there exist some reachable state in which $a$ is applicable. However, this information is apparent from
the public description of public actions, and hence is not private. 

However, there exists a potentially more serious leak of information. Given all the knowledge accumulated during the search process, agents can attempt to recreate some model of other agent's private states, and the possible transitions between these private states. For example, given every public state, agents can see which actions were applicable and which actions where not applicable. If the actions are different, agents can deduce that the states are different. This information can later be used in order to reconstruct a model of the agent's private state using techniques for learning with hidden values or techniques for learning hidden states. Of course, there is no guarantee that this information will be correct or useful, obtaining it requires collaboration between different agents (that need to share which public states they received from the agent), and algorithms for learning in the context of hidden variables/states can be weak. Nevertheless, clearly some information could leak.

The discussion above indicates that careful investigation of information leaks and development of algorithms that have better privacy guarantees is an important avenue for future research. First, it would be interesting to see work that empirically investigates the significance of privacy loss. For example, our empirical results indicate that many problems can be solved quickly using distributed forward search, without expanding too many nodes. Is it possible to build reasonable models of agent's private states in such cases? 

Second, one can develop variants of current algorithms that have stronger privacy preserving properties.
For example, consider the problem of inferring upper bounds on the cost the minimal cost of applying action $a$ in public state $s$. In general, private actions which achieve preconditions of a public action do not have to be applied immediately before \emph{that} public action -- an agent can perform some of the private actions required for a public action before a previous public action. In other words, an agent can ``distribute'' the private cost of a public action between different ``segments'', or parts of the plan between two public actions, making the cost of the first action appear higher and the cost of the second action lower, although with some potential impact on optimality.
In the case of non-optimal search, $g$-values are not disclosed, so this is not an issue.

The above example illustrates a general idea: one can trade-off efficiency for privacy. A similar tradeoff is explored in the area of {\em differential privacy}~\cite{Dwork06}. There, some noise is inserted into a database before statistical queries are evaluated, such that the answer to the statistical query is correct to within some given tolerance, $\epsilon$, yet one cannot infer information about a particular entry in the database (e.g., describing the medical record of an individual). Similarly, in our context, one can consider algorithms in which agents refrain from sending certain public states with some probability, or send it with some random delay, or even possibly, generate bogus, intermediate public states. Such changes are likely to have some impact on running time and solution quality, and these tradeoffs would be interesting to explore. 

We believe that as this area matures, much like in the area of DisCSP, more attention will be given to the problem of precise quantification of privacy and privacy loss. Our work brings us closer to this stage. It offers algorithms for distributed search that start to match that of centralized search, and perhaps more importantly, a general methodology for distributed forward search that respects the natural distributed structure of the system,
that can form a basis for such extensions.

\vskip 0.2in
\bibliography{cites}

\begin{thebibliography}{}

\bibitem[\protect\BCAY{Bernstein, Givan, Immerman,\ \BBA\
  Zilberstein}{Bernstein et~al.}{2002}]{BernsteinGIZ02}
Bernstein, D.~S., Givan, R., Immerman, N., \BBA\ Zilberstein, S.
  \BBOP2002\BBCP.
\newblock \BBOQ The complexity of decentralized control of markov decision
  processes\BBCQ\
\newblock {\Bem Math. Oper. Res.}, {\Bem 27\/}(4), 819--840.

\bibitem[\protect\BCAY{Brafman\ \BBA\ Domshlak}{Brafman\ \BBA\
  Domshlak}{2008}]{Brafman200828}
Brafman, R.~I.\BBACOMMA\  \BBA\ Domshlak, C. \BBOP2008\BBCP.
\newblock \BBOQ From one to many: Planning for loosely coupled multi-agent
  systems\BBCQ\
\newblock In {\Bem ICAPS}, \BPGS\ 28--35.

\bibitem[\protect\BCAY{Brafman, Domshlak, Engel,\ \BBA\ Tennenholtz}{Brafman
  et~al.}{2009}]{BrafmanDET09}
Brafman, R.~I., Domshlak, C., Engel, Y., \BBA\ Tennenholtz, M. \BBOP2009\BBCP.
\newblock \BBOQ Planning games\BBCQ\
\newblock In {\Bem IJCAI}, \BPGS\ 73--78.

\bibitem[\protect\BCAY{Brafman, Domshlak, Engel,\ \BBA\ Tennenholtz}{Brafman
  et~al.}{2010}]{BrafmanDET10}
Brafman, R.~I., Domshlak, C., Engel, Y., \BBA\ Tennenholtz, M. \BBOP2010\BBCP.
\newblock \BBOQ Transferable utility planning games\BBCQ\
\newblock In {\Bem AAAI}.

\bibitem[\protect\BCAY{Burns, Lemons, Ruml,\ \BBA\ Zhou}{Burns
  et~al.}{2010}]{BurnsLRZ10}
Burns, E., Lemons, S., Ruml, W., \BBA\ Zhou, R. \BBOP2010\BBCP.
\newblock \BBOQ Best-first heuristic search for multicore machines\BBCQ\
\newblock {\Bem J. Artif. Intell. Res. (JAIR)}, {\Bem 39}, 689--743.

\bibitem[\protect\BCAY{Chandy\ \BBA\ Lamport}{Chandy\ \BBA\
  Lamport}{1985}]{ChandyL85}
Chandy, K.~M.\BBACOMMA\  \BBA\ Lamport, L. \BBOP1985\BBCP.
\newblock \BBOQ Distributed snapshots: Determining global states of distributed
  systems\BBCQ\
\newblock {\Bem ACM Trans. Comput. Syst.}, {\Bem 3\/}(1), 63--75.

\bibitem[\protect\BCAY{Cox\ \BBA\ Durfee}{Cox\ \BBA\ Durfee}{2005}]{CoxD05}
Cox, J.~S.\BBACOMMA\  \BBA\ Durfee, E.~H. \BBOP2005\BBCP.
\newblock \BBOQ An efficient algorithm for multiagent plan coordination\BBCQ\
\newblock In {\Bem AAMAS}, \BPGS\ 828--835. ACM.

\bibitem[\protect\BCAY{Dwork}{Dwork}{2006}]{Dwork06}
Dwork, C. \BBOP2006\BBCP.
\newblock \BBOQ Differential privacy\BBCQ\
\newblock In {\Bem ICALP (2)}, \BPGS\ 1--12.

\bibitem[\protect\BCAY{Franzin, Rossi, Freuder,\ \BBA\ Wallace}{Franzin
  et~al.}{2004}]{FranzinRFW04}
Franzin, M.~S., Rossi, F., Freuder, E.~C., \BBA\ Wallace, R.~J. \BBOP2004\BBCP.
\newblock \BBOQ Multi-agent constraint systems with preferences: Efficiency,
  solution quality, and privacy loss\BBCQ\
\newblock {\Bem Computational Intelligence}, {\Bem 20\/}(2), 264--286.

\bibitem[\protect\BCAY{Greenstadt, Grosz,\ \BBA\ Smith}{Greenstadt
  et~al.}{2007}]{GreenstadtGS07}
Greenstadt, R., Grosz, B.~J., \BBA\ Smith, M.~D. \BBOP2007\BBCP.
\newblock \BBOQ Ssdpop: improving the privacy of dcop with secret sharing\BBCQ\
\newblock In {\Bem AAMAS}, \BPG\ 171.

\bibitem[\protect\BCAY{Greenstadt, Pearce,\ \BBA\ Tambe}{Greenstadt
  et~al.}{2006}]{GreenstadtPT06}
Greenstadt, R., Pearce, J.~P., \BBA\ Tambe, M. \BBOP2006\BBCP.
\newblock \BBOQ Analysis of privacy loss in distributed constraint
  optimization\BBCQ\
\newblock In {\Bem AAAI}, \BPGS\ 647--653.

\bibitem[\protect\BCAY{Gupta, Segal, Panda, Segev, Schapira, Feigenbaum,
  Rexford,\ \BBA\ Shenker}{Gupta et~al.}{2012}]{GuptaSPSSFRS12}
Gupta, D., Segal, A., Panda, A., Segev, G., Schapira, M., Feigenbaum, J.,
  Rexford, J., \BBA\ Shenker, S. \BBOP2012\BBCP.
\newblock \BBOQ A new approach to interdomain routing based on secure
  multi-party computation\BBCQ\
\newblock In {\Bem HotNets}, \BPGS\ 37--42.

\bibitem[\protect\BCAY{Hansen\ \BBA\ Zilberstein}{Hansen\ \BBA\
  Zilberstein}{2001}]{HansenZ01a}
Hansen, E.~A.\BBACOMMA\  \BBA\ Zilberstein, S. \BBOP2001\BBCP.
\newblock \BBOQ Lao$^{\mbox{*}}$: A heuristic search algorithm that finds
  solutions with loops\BBCQ\
\newblock {\Bem Artif. Intell.}, {\Bem 129\/}(1-2), 35--62.

\bibitem[\protect\BCAY{Helmert}{Helmert}{2006}]{Helmert06}
Helmert, M. \BBOP2006\BBCP.
\newblock \BBOQ The fast downward planning system\BBCQ\
\newblock {\Bem J. Artif. Intell. Res. (JAIR)}, {\Bem 26}, 191--246.

\bibitem[\protect\BCAY{Helmert\ \BBA\ Geffner}{Helmert\ \BBA\
  Geffner}{2008}]{HelmertG08}
Helmert, M.\BBACOMMA\  \BBA\ Geffner, H. \BBOP2008\BBCP.
\newblock \BBOQ Unifying the causal graph and additive heuristics\BBCQ\
\newblock In {\Bem ICAPS}, \BPGS\ 140--147.

\bibitem[\protect\BCAY{Helmert, Haslum,\ \BBA\ Hoffmann}{Helmert
  et~al.}{2007}]{HelmertHH07}
Helmert, M., Haslum, P., \BBA\ Hoffmann, J. \BBOP2007\BBCP.
\newblock \BBOQ Flexible abstraction heuristics for optimal sequential
  planning\BBCQ\
\newblock In {\Bem ICAPS}, \BPGS\ 176--183.

\bibitem[\protect\BCAY{Hoffmann\ \BBA\ Nebel}{Hoffmann\ \BBA\
  Nebel}{2001}]{Hoffmann200114253}
Hoffmann, J.\BBACOMMA\  \BBA\ Nebel, B. \BBOP2001\BBCP.
\newblock \BBOQ The {FF} planning system: fast plan generation through
  heuristic search\BBCQ\
\newblock {\Bem J. Artif. Int. Res.}, {\Bem 14\/}(1), 253--302.

\bibitem[\protect\BCAY{Holte}{Holte}{2010}]{Holte10}
Holte, R.~C. \BBOP2010\BBCP.
\newblock \BBOQ Common misconceptions concerning heuristic search\BBCQ\
\newblock In {\Bem SOCS}.

\bibitem[\protect\BCAY{ICAPS}{ICAPS}{}]{ipc}
ICAPS.
\newblock \BBOQ The international planning competition\BBCQ\
\newblock \url{http://www.plg.inf.uc3m.es/ipc2011-deterministic/}.

\bibitem[\protect\BCAY{Katz\ \BBA\ Domshlak}{Katz\ \BBA\
  Domshlak}{2008}]{KatzD08}
Katz, M.\BBACOMMA\  \BBA\ Domshlak, C. \BBOP2008\BBCP.
\newblock \BBOQ Optimal additive composition of abstraction-based admissible
  heuristics\BBCQ\
\newblock In {\Bem ICAPS}, \BPGS\ 174--181.

\bibitem[\protect\BCAY{Kishimoto, Fukunaga,\ \BBA\ Botea}{Kishimoto
  et~al.}{2009}]{KishimotoFB09}
Kishimoto, A., Fukunaga, A.~S., \BBA\ Botea, A. \BBOP2009\BBCP.
\newblock \BBOQ Scalable, parallel best-first search for optimal sequential
  planning\BBCQ\
\newblock In {\Bem ICAPS}.

\bibitem[\protect\BCAY{L{\'e}aut{\'e}\ \BBA\ Faltings}{L{\'e}aut{\'e}\ \BBA\
  Faltings}{2009}]{LeauteF09}
L{\'e}aut{\'e}, T.\BBACOMMA\  \BBA\ Faltings, B. \BBOP2009\BBCP.
\newblock \BBOQ Privacy-preserving multi-agent constraint satisfaction\BBCQ\
\newblock In {\Bem CSE (3)}, \BPGS\ 17--25.

\bibitem[\protect\BCAY{Maheswaran, Pearce, Bowring, Varakantham,\ \BBA\
  Tambe}{Maheswaran et~al.}{2006}]{MaheswaranPBVT06}
Maheswaran, R.~T., Pearce, J.~P., Bowring, E., Varakantham, P., \BBA\ Tambe, M.
  \BBOP2006\BBCP.
\newblock \BBOQ Privacy loss in distributed constraint reasoning: A
  quantitative framework for analysis and its applications\BBCQ\
\newblock {\Bem Autonomous Agents and Multi-Agent Systems}, {\Bem 13\/}(1),
  27--60.

\bibitem[\protect\BCAY{Meisels}{Meisels}{2007}]{MeiselsBook}
Meisels, A. \BBOP2007\BBCP.
\newblock {\Bem Distributed Search by Constrained Agents: Algorithms,
  Performance, Communication (Advanced Information and Knowledge Processing)}.
\newblock Springer.

\bibitem[\protect\BCAY{M{\'e}ro}{M{\'e}ro}{1984}]{Mero84}
M{\'e}ro, L. \BBOP1984\BBCP.
\newblock \BBOQ A heuristic search algorithm with modifiable estimate\BBCQ\
\newblock {\Bem Artif. Intell.}, {\Bem 23\/}(1), 13--27.

\bibitem[\protect\BCAY{Nilsson}{Nilsson}{1980}]{Nilsson1980}
Nilsson, N.~J. \BBOP1980\BBCP.
\newblock {\Bem Principles of artificial intelligence}.
\newblock Morgan Kaufmann Publishers Inc., San Francisco, CA, USA.

\bibitem[\protect\BCAY{Nissim, Apsel,\ \BBA\ Brafman}{Nissim
  et~al.}{2012}]{NissimAB12}
Nissim, R., Apsel, U., \BBA\ Brafman, R.~I. \BBOP2012\BBCP.
\newblock \BBOQ Tunneling and decomposition-based state reduction for optimal
  planning\BBCQ\
\newblock In {\Bem ECAI}, \BPGS\ 624--629.

\bibitem[\protect\BCAY{Nissim, Brafman,\ \BBA\ Domshlak}{Nissim
  et~al.}{2010}]{NissimBD10}
Nissim, R., Brafman, R.~I., \BBA\ Domshlak, C. \BBOP2010\BBCP.
\newblock \BBOQ A general, fully distributed multi-agent planning
  algorithm\BBCQ\
\newblock In {\Bem AAMAS}, \BPGS\ 1323--1330.

\bibitem[\protect\BCAY{Nissim, Hoffmann,\ \BBA\ Helmert}{Nissim
  et~al.}{2011}]{NissimHH11}
Nissim, R., Hoffmann, J., \BBA\ Helmert, M. \BBOP2011\BBCP.
\newblock \BBOQ Computing perfect heuristics in polynomial time: On
  bisimulation and merge-and-shrink abstraction in optimal planning\BBCQ\
\newblock In {\Bem IJCAI}, \BPGS\ 1983--1990.

\bibitem[\protect\BCAY{Petcu\ \BBA\ Faltings}{Petcu\ \BBA\
  Faltings}{2005}]{PetcuF05}
Petcu, A.\BBACOMMA\  \BBA\ Faltings, B. \BBOP2005\BBCP.
\newblock \BBOQ A scalable method for multiagent constraint optimization\BBCQ\
\newblock In {\Bem IJCAI}, \BPGS\ 266--271.

\bibitem[\protect\BCAY{Petcu, Faltings,\ \BBA\ Parkes}{Petcu
  et~al.}{2008}]{PetcuFP08}
Petcu, A., Faltings, B., \BBA\ Parkes, D.~C. \BBOP2008\BBCP.
\newblock \BBOQ \textsc{M-DPOP}: Faithful distributed implementation of
  efficient social choice problems\BBCQ\
\newblock {\Bem J. Artif. Intell. Res. (JAIR)}, {\Bem 32}, 705--755.

\bibitem[\protect\BCAY{Silaghi\ \BBA\ Mitra}{Silaghi\ \BBA\
  Mitra}{2004}]{SilaghiM04}
Silaghi, M.-C.\BBACOMMA\  \BBA\ Mitra, D. \BBOP2004\BBCP.
\newblock \BBOQ Distributed constraint satisfaction and optimization with
  privacy enforcement\BBCQ\
\newblock In {\Bem IAT}, \BPGS\ 531--535.

\bibitem[\protect\BCAY{Steenhuisen, Witteveen, ter Mors,\ \BBA\
  Valk}{Steenhuisen et~al.}{2006}]{SteenhuisenWMV06}
Steenhuisen, J.~R., Witteveen, C., ter Mors, A., \BBA\ Valk, J. \BBOP2006\BBCP.
\newblock \BBOQ Framework and complexity results for coordinating
  non-cooperative planning agents\BBCQ\
\newblock In {\Bem MATES}, \BPGS\ 98--109.

\bibitem[\protect\BCAY{Szer, Charpillet,\ \BBA\ Zilberstein}{Szer
  et~al.}{2005}]{SzerCZ05}
Szer, D., Charpillet, F., \BBA\ Zilberstein, S. \BBOP2005\BBCP.
\newblock \BBOQ Maa*: A heuristic search algorithm for solving decentralized
  pomdps\BBCQ\
\newblock In {\Bem UAI}, \BPGS\ 576--590.

\bibitem[\protect\BCAY{ter Mors, Valk,\ \BBA\ Witteveen}{ter Mors
  et~al.}{2004}]{MorsVW04}
ter Mors, A., Valk, J., \BBA\ Witteveen, C. \BBOP2004\BBCP.
\newblock \BBOQ Coordinating autonomous planners\BBCQ\
\newblock In {\Bem IC-AI}, \BPGS\ 795--.

\bibitem[\protect\BCAY{ter Mors\ \BBA\ Witteveen}{ter Mors\ \BBA\
  Witteveen}{2005}]{MorsW05}
ter Mors, A.\BBACOMMA\  \BBA\ Witteveen, C. \BBOP2005\BBCP.
\newblock \BBOQ Coordinating self interested autonomous planning agents\BBCQ\
\newblock In {\Bem BNAIC}, \BPGS\ 383--384.

\bibitem[\protect\BCAY{Torre{\~n}o, Onaindia,\ \BBA\ Sapena}{Torre{\~n}o
  et~al.}{2012}]{TorrenoOS12}
Torre{\~n}o, A., Onaindia, E., \BBA\ Sapena, O. \BBOP2012\BBCP.
\newblock \BBOQ An approach to multi-agent planning with incomplete
  information\BBCQ\
\newblock In {\Bem ECAI}, \BPGS\ 762--767.

\bibitem[\protect\BCAY{Vrakas, Refanidis,\ \BBA\ Vlahavas}{Vrakas
  et~al.}{2001}]{VrakasRV01}
Vrakas, D., Refanidis, I., \BBA\ Vlahavas, I.~P. \BBOP2001\BBCP.
\newblock \BBOQ Parallel planning via the distribution of operators\BBCQ\
\newblock {\Bem J. Exp. Theor. Artif. Intell.}, {\Bem 13\/}(3), 211--226.

\bibitem[\protect\BCAY{Yang, Wu,\ \BBA\ Jiang}{Yang et~al.}{2007}]{YangWJ07}
Yang, Q., Wu, K., \BBA\ Jiang, Y. \BBOP2007\BBCP.
\newblock \BBOQ Learning action models from plan examples using weighted
  max-sat\BBCQ\
\newblock {\Bem Artif. Intell.}, {\Bem 171\/}(2-3), 107--143.

\bibitem[\protect\BCAY{Yao}{Yao}{1982}]{Yao82b}
Yao, A. C.-C. \BBOP1982\BBCP.
\newblock \BBOQ Protocols for secure computations (extended abstract)\BBCQ\
\newblock In {\Bem FOCS}, \BPGS\ 160--164.

\bibitem[\protect\BCAY{Yao}{Yao}{1986}]{Yao86}
Yao, A. C.-C. \BBOP1986\BBCP.
\newblock \BBOQ How to generate and exchange secrets (extended abstract)\BBCQ\
\newblock In {\Bem FOCS}, \BPGS\ 162--167.

\bibitem[\protect\BCAY{Yokoo, Durfee, Ishida,\ \BBA\ Kuwabara}{Yokoo
  et~al.}{1998}]{YokooDIK98}
Yokoo, M., Durfee, E.~H., Ishida, T., \BBA\ Kuwabara, K. \BBOP1998\BBCP.
\newblock \BBOQ The distributed constraint satisfaction problem: Formalization
  and algorithms\BBCQ\
\newblock {\Bem IEEE Trans. Knowl. Data Eng.}, {\Bem 10\/}(5), 673--685.

\bibitem[\protect\BCAY{Yokoo, Suzuki,\ \BBA\ Hirayama}{Yokoo
  et~al.}{2002}]{YokooSH02}
Yokoo, M., Suzuki, K., \BBA\ Hirayama, K. \BBOP2002\BBCP.
\newblock \BBOQ Secure distributed constraint satisfaction: Reaching agreement
  without revealing private information\BBCQ\
\newblock In {\Bem CP}, \BPGS\ 387--401.

\end{thebibliography}
\bibliographystyle{theapa}

\end{document}